\documentclass[journal]{IEEEtran} 
\usepackage{graphicx}  \usepackage{times} \usepackage{amsmath}  \usepackage{amssymb}
\usepackage{algorithmic}  \usepackage{algorithm} \usepackage{epsfig} \usepackage{epstopdf}
\usepackage{multirow} \usepackage{subfigure} \usepackage{caption} \usepackage{cite}

\usepackage{amsthm}
\usepackage{lipsum} \usepackage{xcolor} \usepackage{hyperref} 
\usepackage{gensymb}
\usepackage{bbm}

\theoremstyle{remark}
\newtheorem{theorem}{Theorem}
\newtheorem{lemma}{Lemma}
\newtheorem*{upper_bnd}{Upper Bound of SparseMax Operation}
\newtheorem*{lower_bnd}{Lower Bound of SparseMax Operation}

\title{\bf Sparse Markov Decision Processes with\\
Causal Sparse Tsallis Entropy Regularization\\
for Reinforcement Learning}

\author{Kyungjae Lee, Sungjoon Choi, and Songhwai Oh% <-this % stops a space
\thanks{ K. Lee, S. Choi, and S. Oh are with 
       the Department of Electrical and Computer Engineering and ASRI,  
       Seoul National University, Seoul 08826, Korea
	(e-mail: \{kyungjae.lee, sungjoon.choi, songhwai.oh\}@cpslab.snu.ac.kr).
}
}

\begin{document}
\maketitle

\begin{abstract}
In this paper, a sparse Markov decision process (MDP) with novel causal sparse Tsallis entropy regularization is proposed.
The proposed policy regularization induces a sparse and multi-modal optimal policy distribution of a sparse MDP.
The full mathematical analysis of the proposed sparse MDP is provided.
We first analyze the optimality condition of a sparse MDP.
Then, we propose a sparse value iteration method which solves a sparse
MDP and then prove the convergence and optimality of sparse value
iteration using the Banach fixed point theorem. 
The proposed sparse MDP is compared to soft MDPs which utilize causal
entropy regularization. 
We show that the performance error of a sparse MDP has a constant bound,
while the error of a soft MDP increases logarithmically with
respect to the number of actions, where this performance error is
caused by the introduced regularization term. 
In experiments, 
we apply sparse MDPs to reinforcement learning problems.
The proposed method outperforms existing methods in terms of the
convergence speed and performance.
\end{abstract}

\section{Introduction}

% 1. Markov decision processes are super important
% Both RL aim to find policy
\IEEEPARstart{M}arkov decision processes (MDPs) have been widely used as a mathematical framework 
to solve stochastic sequential decision problems, 
such as autonomous driving \cite{brechtel2014probabilistic}, path planning \cite{ragi2013path}, and quadrotor control \cite{hwangbo2017control}.
In general, the goal of an MDP is to find the optimal policy function which maximizes the expected return.
The expected return is a performance measure of a policy function
and it is often defined as the expected sum of discounted rewards.
An MDP is often used to formulate reinforcement learning (RL) \cite{kober2013reinforce}, 
which aims to find the optimal policy without the explicit specification of stochasticity of an environment, 
and inverse reinforcement learning (IRL) \cite{ng2000algorithms}, 
whose goal is to search the proper reward function that can explain
the behavior of an expert who follows the underlying optimal policy.

While the optimal solution of an MDP is a deterministic policy,
it is not desirable to apply an MDP to the problems with multiple optimal actions.
In perspective of RL, 
the knowledge of multiple optimal actions makes it possible to cope with unexpected situations.
For example, suppose that an autonomous vehicle has multiple optimal routes to reach a given goal. 
If a traffic accident occurs at the currently selected optimal route,
it is possible to avoid the accident by choosing another safe optimal
route without additional computation of a new optimal route.
For this reason, it is more desirable to learn all possible optimal
actions in terms of robustness of a policy function.
In perspective of IRL, since the experts often make multiple decisions
in the same situation, a deterministic policy has a limitation in
expressing the expert's behavior. 
For this reason, it is indispensable to model the policy function of
an expert as a multi-modal distribution.
These reasons give a rise to the necessity of a multi-modal policy model.

% 1.1 Entropy regularization is used a lot, Importance of stochastic policy 
%The optimal solution of an MDP is a deterministic policy. 
%However, a deterministic policy is not desirable when there are
%multiple optimal actions.
%{\bf [But is it the only reason that a deterministic policy is not good?]}
%{\color{blue}
%The knowledge of multiple optimal actions makes it possible to cope
%with unexpected situations.  
%For example, suppose that an autonomous vehicle has multiple optimal
%routes to reach a given goal. 
%If a traffic accident occurs at the currently selected optimal route,
%it is possible to avoid the accident by choosing another safe optimal
%route without additional computing a new optimal route. 
%For reason, it is more desirable to learn all possible optimal actions
%in perspective of robustness of policy function.}
%{\bf [The argument is not very strong and convincing. Try to find supporting arguments from literature.]}
%This gives a rise to the necessity of a muti-modal policy model.
% 2. MaxEnt
In order to address the issues with a deterministic policy function, a
causal entropy regularization method has been utilized 
\cite{Haarnoja2017, Heess2012, schulman2017equivalence, tokic2011value, vamplew2017softmax}.
This is mainly due to the fact that the optimal solution of an MDP
with causal entropy regularization becomes a softmax distribution
of state-action values $Q(s, a)$, 
i.e., $\pi(a|s) = \frac{\exp(Q(s,a))}{\sum_{a'}\exp(Q(s,a'))}$,
which is often referred to as a soft MDP \cite{bloem2014infinite}.
% What we propose / Why we propose 
While a softmax distribution has been widely used to 
model a stochastic policy, it has a weakness
in modeling a policy function when the number of actions is large.
In other words, the policy function modeled by a softmax distribution
is prone to assign non-negligible probability mass to non-optimal
actions even if state-action values of these actions are dismissible.  
This tendency gets worse as the number of actions increases as
demonstrated in Figure \ref{fig:concept}. 

% WHAT WE PROPOSE!!
In this paper, we propose a sparse MDP by
presenting a novel causal sparse Tsallis entropy regularization method,
which can be interpreted as a special case of Tsallis generalized
entropy \cite{Tsallis1988possible}. 
The proposed regularization method has a unique property in that 
the resulting policy distribution becomes a sparse distribution. 
In other words, the supporting action set which has a non-zero probability mass contains a sparse subset of the action space. 

% Contribution
We provide a full mathematical analysis about the proposed sparse MDP.
We first derive the optimality condition of a sparse MDP, which is
named as a sparse Bellman equation. 
We show that the sparse Bellman equation is an approximation of the
original Bellman equation. 
Interestingly, we further find the connection between the optimality
condition of a sparse MDP and the probability simplex projection
problem \cite{wang2013projection}. 
We present a sparse value iteration method for solving a sparse MDP
problem, where the optimality and convergence are proven using the
Banach fixed point theorem \cite{smart1980fixed}.
We further analyze the performance gaps of the expected return of the
optimal policies obtained by a sparse MDP and a soft MDP compared to
that of the original MDP. 
In particular, we prove that the performance gap between the proposed
sparse MDP and the original MDP has a constant bound as the number of
actions increases, whereas the performance gap between a soft MDP and
the original MDP grows logarithmically. 
From this property, sparse MDPs have benefits over soft MDPs when it
comes to solving problems in robotics with a continuous action space. 

% What experiments we have doen
To validate effectiveness of a sparse MDP,
we apply the proposed method to the exploration strategy and the
update rule of Q-learning and compare to the $\epsilon$-greedy method
and softmax policy \cite{tokic2011value}. 
The proposed method is also compared to the deep deterministic policy
gradient (DDPG) method \cite{lillicrap2015continuous}, which is
designed to operate in a continuous action space without
discretization. 
The proposed method shows the state of the art performance compared to
other methods as the discretization level of an action space increases. 

\section{Background}

\subsection{Markov Decision Processes}

A Markov decision process (MDP) has been widely used to formulate a sequential decision making problem.
An MDP can be characterized by a tuple $\mathbf{M} = \{\mathcal{S}, \mathcal{F}, \mathcal{A}, d, T, \gamma, \mathbf{r} \}$, 
where $\mathcal{S}$ is the state space, $\mathcal{F}$ is the corresponding feature space, $\mathcal{A}$ is the action space,
$d(s)$ is the distribution of an initial state,
$T(s'|s,a)$ is the transition probability from $s\in\mathcal{S}$ to $ s'\in\mathcal{S}$ by taking $a\in\mathcal{A}$,
$\gamma \in (0,1)$ is a discount factor, and $\mathbf{r}$ is the reward function.
The objective of an MDP is to find a policy which maximize $\mathbb{E}\left[\sum_{t=0}^{\infty} \gamma^{t} r(s_t,a_t) \middle| \pi, d, T \right]$,
where policy $\pi$ is a mapping from the state space to the action space.
For notational simplicity, we denote the expectation of a discounted
summation of function $f(s,a)$, 
i.e., $\mathbb{E}[\sum_{t=0}^{\infty}\gamma^{t} f(s_t,a_t) |\pi,d,T]$, 
by $\mathbb{E}_{\pi}[f(s,a)]$, where $f(s,a)$ is a function of state
and action, such as a reward function $\mathbf{r}(s,a)$ or an
indicator function $\mathbbm{1}_{\{s = s'\}}$. 
We also denote the expectation of a discounted summation of function
$f(s,a)$ conditioned on the initial state, 
i.e., $\mathbb{E}[\sum_{t=0}^{\infty}\gamma^{t} f(s_t,a_t) |\pi,s_{0}=s,T]$, 
by $\mathbb{E}_{\pi}[f(s,a)|s_{0}=s]$.
Finding an optimal policy for an MDP can be formulated as follows:
\begin{eqnarray}\label{eqn:mdp}
\begin{aligned}
& \underset{\pi}{\text{maximize}}
& & \mathbb{E}_{\pi}\left[\mathbf{r}(s_t,a_t)\right]  \\
& \text{subject to}
& &\forall \, s \;\; \sum_{a'} \pi(a'|s) = 1,\;\;  \forall \, s,a  \;\; \pi(a'|s)\ge0.
\end{aligned}
\end{eqnarray}
The necessary condition for the optimal solution of (\ref{eqn:mdp}) is
called the Bellman equation.
The Bellman equation is derived from the Bellman's optimality principal as follows:
\begin{align}
Q_{\pi}(s,a) &= \mathbf{r}(s,a) + \gamma \sum_{s'}V_{\pi}(s')T(s'|s,a) \notag \\
V_{\pi}(s) &=  \max_{a'} Q(s,a')\notag \\
\pi(s) &= \arg\max_{a'} Q(s,a'), \notag
\end{align}
where $V_{\pi}(s)$ is a value function of $\pi$, which is the expected
sum of discounted rewards when the initial state is given as $s$,
and $Q_{\pi}(s,a)$ is a state-action value function of $\pi$, which is
the expected sum of discounted rewards when the initial state and
action are given as $s$ and $a$, respectively. 
%where $K$ is the number of actions which have the maximum action value, i.e. $Q(s,a) = \max_{a'} Q(s,a')$, 
%and is normalization factor of a policy function.
%In this paper, we are refer to the optimal policy of an original MDP as a \textit{max} policy.
%Note that, in many case, the action with the maximum action value is unique, i.e. $K=1$.
Note that the optimal solution is a deterministic function, which is referred to as a deterministic policy.

%However, if there exist multiple objectives in a given task, a deterministic policy or max policy is improper to represent multi-modality.
%This gives a raise to necessity of a multi-modal policy distribution.

\begin{figure}[t]
\centering
\subfigure	[Reward map and action values at state $s$.]{
  \centering
  \includegraphics[width=.44\textwidth]{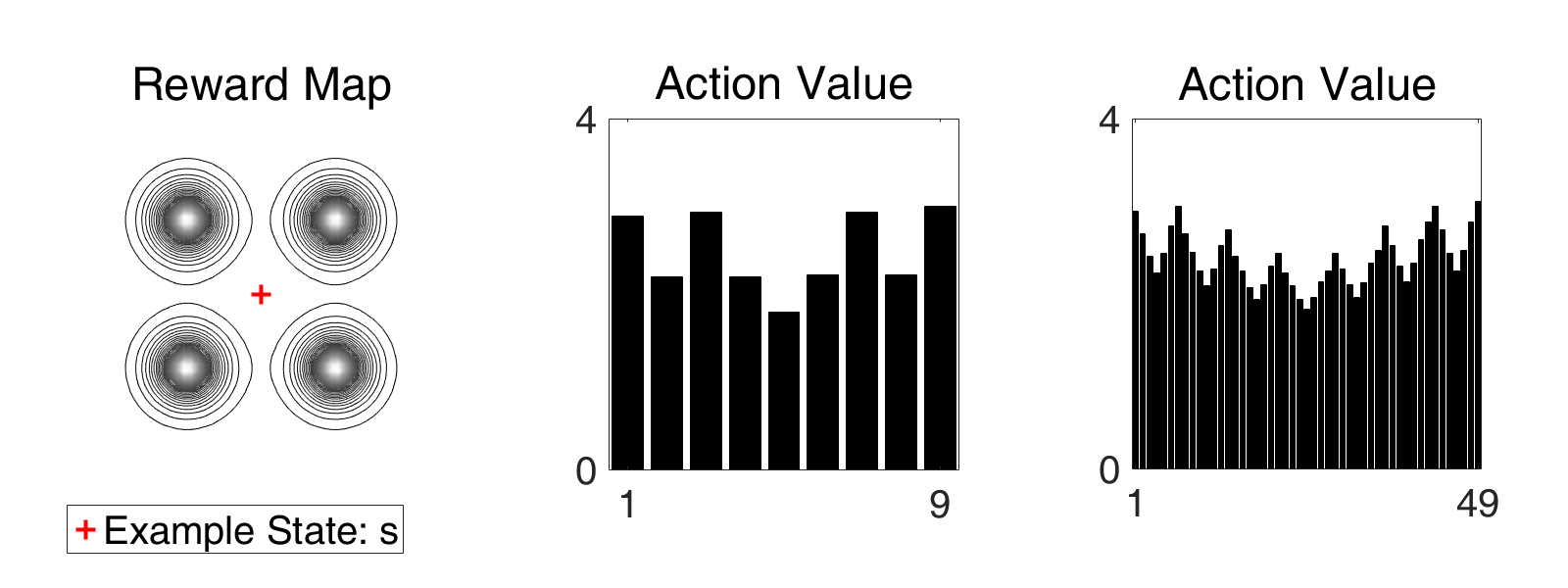}}
\subfigure[Proposed policy model and value differences (darker is better).]{
  \centering
  \includegraphics[width=.48\textwidth]{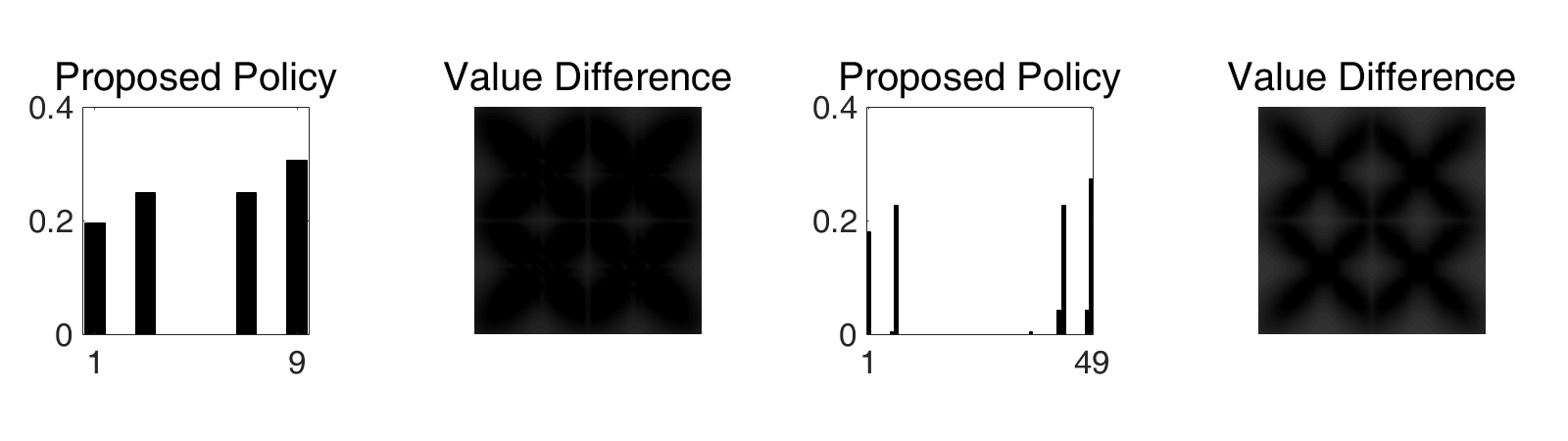}}
\subfigure[Softmax policy model and value differences (darker is better).]{
  \centering
  \includegraphics[width=.48\textwidth]{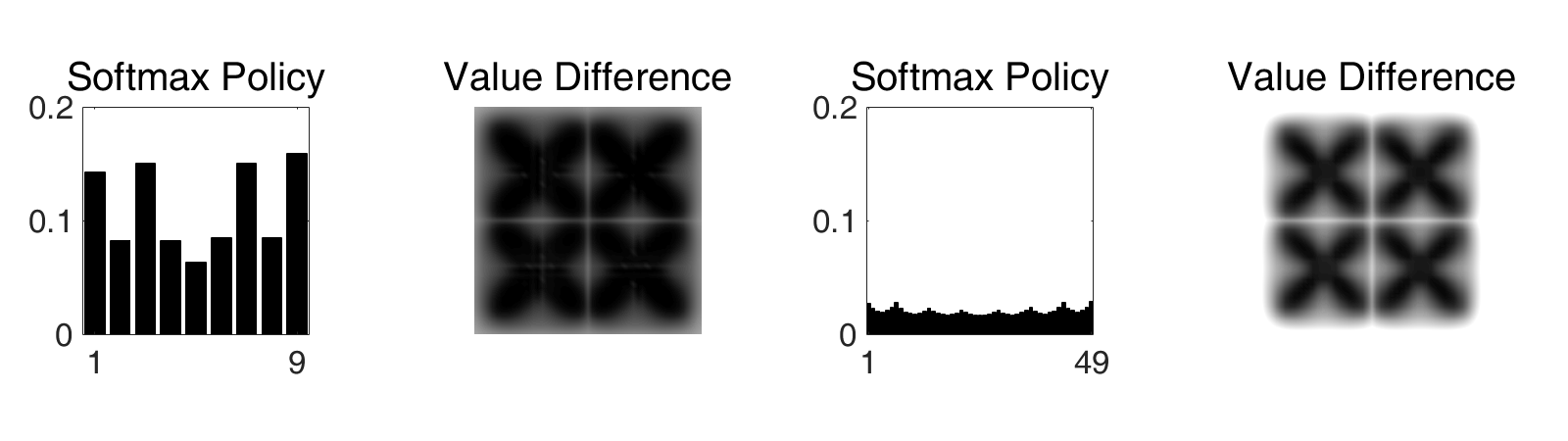}}
\caption{%\small
A 2-dimensional multi-objective environment with point mass dynamics.
The state is a location and the action is a velocity bounded with
$[-3,3] \times [-3,3]$. 
(a) The left figure shows the reward map with four maxima (multiple
objectives). 
The action space is discretized into two levels: $9$ (low resolution)
and $49$ (high resolution). 
The middle (resp., right) figure shows the optimal action value at
state $s$ indicated as red cross point when the number of action is
$9$ (resp., $49$). 
(b) The first and third figure indicate the proposed policy
distributions at state $s$ induced by the action values in (a). 
The second and fourth figure show a map of the performance difference
between the proposed policy and the optimal policy 
at each state when the number of action is $9$ and $49$, respectively.
The larger the error, the brighter the color of the state.
(c) All figures are obtained in the same way as (b) by replacing the
proposed policy with a softmax policy. 
This example shows that the proposed policy model is less affected
when the number of actions increases.
}
\label{fig:concept}
\end{figure}

\subsection{Entropy Regularized Markov Decision Processes}

%%%%%%%%%%%%%%%
%%% What is MDP?
%%%%%%%%%%%%%%%
In order to obtain a multi-modal policy function, an
entropy-regularized MDP, also known as a soft MDP, has been widely
used 
\cite{tokic2011value,bloem2014infinite,schulman2017equivalence,vamplew2017softmax}.
In a soft MDP, causal entropy regularization over $\pi$ is introduced
to obtain a multi-modal policy distribution, i.e., $\pi(a|s)$.
Since causal entropy regularization penalizes a deterministic
distribution, it makes an optimal policy of a soft MDP to be a
softmax distribution. 
%%%%%%%%%%%%%%%
%%% Formulation: Infinite horizon discounted reward sum setting
%%%%%%%%%%%%%%%
A soft MDP is formulated as follows:
\begin{eqnarray}\label{eqn:soft_mdp}
%\small
\begin{aligned}
& \underset{\pi}{\text{maximize}}
& & \mathbb{E}_{\pi}\left[\mathbf{r}(s_t,a_t)\right] + \alpha H(\pi)  \\
& \text{subject to}
& &\forall \, s \;\; \sum_{a'} \pi(a'|s) = 1,\;\;  \forall \, s,a  \;\; \pi(a'|s)\ge0,
\end{aligned}
\end{eqnarray}
where $H(\pi) \triangleq \mathbb{E}_{\pi}\left[-\log(\pi(a_t|s_t))\right]$ 
is a $\gamma$-discounted causal entropy and $\alpha$ is a regularization coefficient.
%%%%%%%%%%%%%%%
%%% Soft Bellman Equation and Soft Value Iteration
%%%%%%%%%%%%%%%
This problem (\ref{eqn:soft_mdp}) has been extensively studied in
\cite{Haarnoja2017,bloem2014infinite,schulman2017equivalence}. 
In \cite{bloem2014infinite}, a soft Bellman equation and the optimal
policy distribution are derived from the Karush Kuhn Tucker (KKT)
conditions as follows:
\begin{eqnarray*}
%\small 
\begin{aligned}
Q_{\pi}^{soft}(s,a) &= r(s,a) + \gamma\sum_{s'} V_{\pi}^{soft}(s')T(s'|s,a)\\
V^{soft}_{\pi}(s) &= \alpha\log\left(\sum_{a'}\exp\left(\frac{Q_{\pi}^{soft}(s,a')}{\alpha}\right)\right)\\
\pi(a|s) &= \frac{\exp\left(\frac{Q_{\pi}^{soft}(s,a)}
{\alpha}\right)}{\sum_{a'}\exp\left(\frac{Q_{\pi}^{soft}(s,a')}{\alpha}\right)},
\end{aligned}
\end{eqnarray*}
where 
\begin{eqnarray*}
%\small
\begin{aligned}
V^{soft}_{\pi}(s) &= \mathbb{E}_{\pi}\left[r(s_t,a_t) - \alpha\log(\pi(a_t|s_t)) \middle| s_0 = s \right]\\
Q^{soft}_{\pi}(s,a) &= \mathbb{E}_{\pi}\left[r(s_t,a_t) - \alpha\log(\pi(a_t|s_t)) \middle| s_0 = s, a_0 = a \right].
\end{aligned}
\end{eqnarray*}
$V_{\pi}^{soft}(s)$ is a soft value of $\pi$ indicating the
expected sum of rewards including the entropy of a policy, obtained by
starting at state $s$ and $Q_{\pi}^{soft}(s,a)$ is a soft
state-action value of $\pi$, which is the expected sum of rewards
obtained by starting at state $s$ by taking action $a$. 
Note that the optimal policy distribution is a softmax distribution.
In \cite{bloem2014infinite}, a soft value iteration method is also
proposed and the optimality of soft value iteration is proved. 
By using causal entropy regularization, the optimal policy
distribution of a soft MDP is able to represent a multi-modal
distribution. 

The causal entropy regularization has an effect of making the
resulting policy of a soft MDP closer to a uniform distribution as the
number of actions increases. 
To handle this issue, we propose a novel regularization method whose
resulting policy distribution still has multiple modes (a stochastic
policy) but the performance loss is less than a softmax policy
distribution.  

\section{Sparse Markov Decision Processes}
%%%%%%%%%%%%%%%
%%% Formulation and What is new?
%%%%%%%%%%%%%%%

We propose a sparse Markov decision process by introducing a novel
causal sparse Tsallis entropy regularizer:
\begin{eqnarray*}
%\small
\begin{aligned}
	W(\pi) &\triangleq 
		\mathbb{E}\left[
			\sum_{t=0}^{\infty} \gamma^{t}\frac{1}{2}(1-\pi(a_t|s_t))\middle|\pi,d,T\right
			]\\
	&=\mathbb{E}_{\pi}\left[\frac{1}{2}(1-\pi(a|s))\right].
\end{aligned}
\end{eqnarray*}
By adding $W(\pi)$ to the objective function of (\ref{eqn:mdp}), we
aim to solve the following optimization problem: 
\begin{eqnarray}
\begin{aligned}\label{eqn:sps_mdp}
	& \underset{\pi}{\text{maximize}}
	& & \mathbb{E}_{\pi}\left[\mathbf{r}(s,a)\right] + \alpha W(\pi)  \\
	& \text{subject to}
	& &\forall \, s \;\; \sum_{a'} \pi(a'|s) = 1,\;\; \forall \, s,a  \;\; \pi(a'|s)\ge0,
\end{aligned}
\end{eqnarray}
where $\alpha > 0$ is a regularization coefficient.
% 1. Sparse Bellman equation: necessity condition of SMDP
% 2. Connection with probability simplex projection
% 3. Optimality guarantee 
We will first derive the sparse Bellman equation from the necessary condition 
of (\ref{eqn:sps_mdp}). 
Then by observing the connection between the sparse Bellman equation
and the probability simplex projection, we show that the optimal
policy becomes a \textit{sparsemax} distribution, where the sparsity
can be controlled by $\alpha$.   
In addition, we present a sparse value iteration algorithm where the
optimality is guaranteed using the Banach's fixed point theorem.
The detailed derivations of lemmas and theorems in this paper can be found
in Appendix \ref{sec:analysis}. 

\subsection{Notations and Properties}

We first introduce notations and properties used in the paper.
In Table \ref{tbl:notations}, all notations and definitions are summarized.
The utility, value, state visitation can be compactly expressed as
below in terms of vectors and matrices:
\begin{eqnarray*}
\begin{aligned}
J^{sp}_{\pi} &= d^{\intercal}G_{\pi}^{-1}r^{sp}_{\pi},& \;\;\;V^{sp}_{\pi} &= G_{\pi}^{-1}r^{sp}_{\pi}\\
J^{soft}_{\pi} &= d^{\intercal}G_{\pi}^{-1}r^{soft}_{\pi},& \;\;\;V^{soft}_{\pi} &= G_{\pi}^{-1}r^{soft}_{\pi}, \;\; \rho_{\pi} = d^{\intercal}G_{\pi}^{-1}
\end{aligned}
\end{eqnarray*}
where $x^{\intercal}$ is the transpose of vector $x$, 
$G_{\pi} = (I - \gamma T_{\pi})$,
$sp$ indicates a sparse MDP problem
and $soft$ indicates a soft MDP problem.

\begin{table*}[h] 
\centering
\small
\begin{tabular}{|p{.1\textwidth}|p{.4\textwidth}|p{.4\textwidth}|}
\hline
Terms  & sparse MDP & soft MDP \\ \hline
Utility                 &$J^{sp}_{\pi}\triangleq\mathbb{E}_{\pi}\left[\mathbf{r}(s',a') + \frac{\alpha}{2}(1-\pi(a'|s')) \right]$\newline$=\sum_{s}d(s)V_{\pi}^{sp}(s)=\sum_{s}\mathbf{r}_{\pi}^{sp}(s)\rho_{\pi}(s)$ &$J^{soft}_{\pi} \triangleq\mathbb{E}_{\pi}\left[\mathbf{r}(s',a') - \alpha\log(\pi(a'|s')) \right]$\newline$=\sum_{s}d(s)V_{\pi}^{soft}(s)=\sum_{s}\mathbf{r}_{\pi}^{soft}(s)\rho_{\pi}(s)$ \\ \hline
Value                   &$V^{sp}_{\pi}(s)$\newline $\triangleq\mathbb{E}_{\pi}\left[\mathbf{r}(s',a') + \frac{\alpha}{2}(1-\pi(a'|s'))\middle| s_0 = s \right]$\newline$=\mathbf{r}_{\pi}^{sp}(s) + \gamma \sum_{s'}V_{\pi}^{sp}(s')T_{\pi}(s'|s)$&$V^{soft}_{\pi}(s)$\newline$\triangleq\mathbb{E}_{\pi}\left[\mathbf{r}(s',a') - \alpha\log(\pi(a'|s'))\middle| s_0 = s \right]$\newline$=\mathbf{r}_{\pi}^{soft}(s) + \gamma \sum_{s'}V_{\pi}^{soft}(s')T_{\pi}(s'|s)$\\ \hline
Action value            &$Q^{sp}_{\pi}(s,a) \triangleq$\newline$\mathbf{r}(s,a) + \gamma \sum_{s'}V^{sp}_{\pi}(s')T(s'|s,a)$& $Q^{soft}_{\pi}(s,a) \triangleq$\newline$\mathbf{r}(s,a) + \gamma \sum_{s'}V^{soft}_{\pi}(s')T(s'|s,a)$\\ \hline
 Expected State Reward   &$\mathbf{r}^{sp}_{\pi}(s)\triangleq$\newline $\sum_{a'} \left(\mathbf{r}(s,a') + \frac{\alpha}{2}(1-\pi(a'|s))\right)\pi(a'|s)$&$\mathbf{r}^{soft}_{\pi}(s)\triangleq$\newline $\sum_{a'} \left(\mathbf{r}(s,a') - \alpha\log(\pi(a'|s))\right)\pi(a'|s)$ \\ \hline
Policy Regularization & $W(\pi) \triangleq \mathbb{E}_{\pi}\left[
			\frac{1}{2}(1-\pi(a|s))\right]$\newline
			$=\sum_{s,a}\frac{1}{2}(1-\pi(a|s))\pi(a|s)\rho(s)$  & $H(\pi) = \mathbb{E}_{\pi}[-\pi(a|s)\log(\pi(a|s))]$\newline
$=\sum_{s,a}-\pi(a|s)\log(\pi(a|s))\rho(s)$ \\ \hline
Max Approximation & $\text{spmax}(z) \triangleq \frac{1}{2}\sum_{i = 1}^{K}\left(z_{(i)}^2 - \tau(z)^2\right) + \frac{1}{2}$& $\text{logsumexp}\left(z\right)\triangleq \log \sum_{i} \exp(z_i)$ \\ \hline
Value Iteration Operator & $ U^{sp}(x)(s) = \alpha\text{spmax}\left(\frac{r(s,\cdot) + \gamma\sum_{s'} x(s')T(s'|s,\cdot)}{\alpha}\right)$& $ U^{soft}(x)(s) = \alpha\text{logsumexp}\left(\frac{r(s,\cdot) + \gamma\sum_{s'} x(s')T(s'|s,\cdot)}{\alpha}\right)$ \\ \hline
 State Visitation        & \multicolumn{2}{c|}{$\rho_{\pi}(s) \triangleq \mathbb{E}_{\pi}\left[\mathbbm{1}_{\{s' = s\}} \right] = d(s) + \gamma\sum_{s',a'}T(s|s',a')\rho_{\pi}(s',a')$}  \\ \hline
 State Action Visitation & \multicolumn{2}{c|}{$\rho_{\pi}(s,a) \triangleq\mathbb{E}_{\pi}\left[\mathbbm{1}_{\{s' = s, a' = a\}} \right] = \pi(a|s)d(s) + \gamma\sum_{s',a'}\pi(a|s)T(s|s',a')\rho_{\pi}(s',a')$} \\ \hline
 Transition Probability given $\pi$ & \multicolumn{2}{c|}{$T_{\pi}(s'|s) \triangleq \sum_{a} T(s'|s,a)\pi(a|s)$} \\ \hline
\end{tabular}
\caption{Notations and Properties}
\label{tbl:notations}
\end{table*}

\subsection{Sparse Bellman Equation from Karush-Kuhn-Tucker conditions}
%%%%%%%%%%%%%%%
%%% Write down KKT
%%%%%%%%%%%%%%%

The sparse Bellman equation can be derived from the necessary conditions
of an optimal solution of a sparse MDP.
We carefully investigate the Karush Kuhn Tucker (KKT) conditions, which indicate necessary conditions for a solution to be optimal
when some regularity conditions about the feasible set are satisfied.
The feasible set of a sparse MDP satisfies linearity constraint
qualification \cite{ye2000constraint} since the feasible set consists
of linear afine functions. 
In this regards, the optimal solution of a sparse MDP necessarily satisfy KKT conditions as follows.

%%%%%%%%%%%%%%%
%%% Sparse Bellman Equation
%%%%%%%%%%%%%%%
\begin{theorem}\label{thm:sp_bellman_eqn}
If a policy distribution $\pi$ is the optimal solution of a sparse MDP
(\ref{eqn:sps_mdp}), then $\pi$ and the corresponding sparse value
function $V_{\pi}^{sp}$ necessarily satisfy following equations for
all state and action pairs: 
{\footnotesize
\begin{align}
Q_{\pi}^{sp}(s,a) &= \mathbf{r}(s,a) + \gamma \sum_{s'}V_{\pi}^{sp}(s')T(s'|s,a) \notag \\
V_{\pi}^{sp}(s) &= \alpha\left[\frac{1}{2}\sum_{a\in S(s)}\left(\left(\frac{Q_{\pi}^{sp}(s,a)}{\alpha}\right)^{2} - \tau \left(\frac{Q_{\pi}^{sp}(s,\cdot)}{\alpha}\right)^{2}\right) + \frac{1}{2}\right] \notag \\
\pi(a|s) &= \max\left(\frac{Q_{\pi}^{sp}(s,a)}{\alpha} - \tau
\left(\frac{Q_{\pi}^{sp}(s,\cdot)}{\alpha}\right),0\right),
\label{eqn:sps_bellman}
\end{align}
}
where $\tau\left(\frac{Q_{\pi}^{sp}(s,\cdot)}{\alpha}\right) = \frac{\sum_{a\in S(s)}\frac{Q_{\pi}^{sp}(s,a)}{\alpha} - 1}{K_s}$, $S(s)$ is a set of actions 
satisfying $1 + i\frac{Q_{\pi}^{sp}(s,a_{(i)})}{\alpha}>\sum_{j=0}^{i}\frac{Q_{\pi}^{sp}(s,a_{(j)})}{\alpha}$
with $a_{(i)}$  indicating the action with the $i$th largest action value $Q_{\pi}^{sp}(s,a_{(i)})$,
and $K_s$ is the cardinality of $S(s)$.
\end{theorem}

The full proof of Theorem \ref{thm:sp_bellman_eqn} is provided in Appendix \ref{subsec:sp_bellman_eq}.
The proof depends on the KKT condition where the derivative of a
Lagrangian objective function with respect to policy $\pi(a|s)$
becomes zero at the optimal solution, the stationary condition. 
From (\ref{eqn:sps_bellman}), it can be shown that the optimal
solution obtained from the sparse MDP assigns zero probability 
to the action whose action value $Q^{sp}(s,a)$ is below the threshold
$\tau\left(\frac{Q_{\pi}^{sp}(s,\cdot)}{\alpha}\right)$ 
and the optimal policy assigns positive probability to near optimal
actions in proportion to their action values, where the threshold
$\tau\left(\frac{Q_{\pi}^{sp}(s,\cdot)}{\alpha}\right)$ determines the
range of near optimal actions.
This property makes the optimal policy to have a sparse distribution and
prevents the performance drop caused by assigning non-negligible
positive probabilities to non-optimal actions, which often occurs in a soft MDP.

From the definitions of $S(s)$ and $\pi(a|s)$, we can further observe 
an interesting connection between the sparse Bellman equation and the
probability simplex projection problem \cite{wang2013projection}. 

\subsection{Probability Simplex Projection and SparseMax Operation}
%%%%%%%%%%%%%%%
%%% Relation to Simplex Probjection problem and well known solution
%%%%%%%%%%%%%%%
The probability simplex projection \cite{wang2013projection} is a well
known problem of projecting a $d$-dimensional vector into a $d-1$
dimensional probability simplex in an Euclidean metric sense.  
A probability simplex projection problem is defined as follows:
\begin{eqnarray}
\small
\begin{aligned}\label{eqn:simplex_prob}
& \underset{p}{\text{minimize}}
& & \frac{1}{2}||p - z||^2_2 \\
& \text{subject to}
& &\;\; \sum_{i=1}^{d} p_i = 1, \;\;\; p_i\ge0,\; \forall i = 1,\cdots,d,
\end{aligned}
\end{eqnarray}
where $z$ is a given $d$-dimensional vector, $d$ is the dimension of
$p$ and $z$, and $p_{i}$ is the $i$th element of $p$.
Let $z_{(i)}$ be the $i$th largest element of $z$ and $\text{supp}(z)$
be the supporting set of the optimal solution as defined by 
$\text{supp}(z) = \{z_{(i)}| 1+iz_{(i)} > \sum_{j=1}^{i}z_{(j)}\}$. 
It is a well known fact that the problem (\ref{eqn:simplex_prob}) has
a closed form solution which is $p_{i}^{*}(z) = \max(z_{i} - \tau(z), 0)$, 
where $i$ indicates the $i$th dimension,
$p_{i}^{*}(z)$ is the $i$th element of the optimal solution for fixed $z$,
and $\tau(z) = \frac{\sum_{i = 1}^{K}z_{(i)} - 1}{K}$
with $K = |\text{supp}(z)|$ \cite{wang2013projection,martins2016softmax}.

Interestingly, the optimal solution $p^{*}(\cdot)$, $\tau(\cdot)$ and
the supporting set $\text{supp}(\cdot)$ of (\ref{eqn:simplex_prob})  
can be precisely matched to those of the sparse Bellman equation (\ref{eqn:sps_bellman}).
From this observation, it can be shown that the optimal policy
distribution of a sparse MDP is the projection of
$Q_{\pi}^{sp}(s,\cdot)$ into a probability simplex.
Note that we refer $p^{*}(\cdot)$ as a \textit{sparsemax} distribution.

%%%%%%%%%%%%%%%
%%% Sparsemax Operation
%%%%%%%%%%%%%%%
More surprisingly, $V_{\pi}^{sp}$ can be represented as an
approximation of the \textit{max} operation derived from $p^{*}(z)$. 
A differentiable approximation of the \textit{max} operation is defined as follows:
\begin{equation}\label{eqn:spmax_dfn}
\text{spmax}(z) \triangleq \frac{1}{2}\sum_{i = 1}^{K}\left(z_{(i)}^2 - \tau(z)^2\right) + \frac{1}{2}
\end{equation}
We call $\text{spmax}(z)$ as \textit{sparsemax}.
In \cite{martins2016softmax}, it is proven that $\text{spmax}(z)$ is
an indefinite integral of $p^{*}(z)$, 
i.e., $\text{spmax}(z) = \int \left(p^{*}(z)\right)^{\intercal}\bold{d}z + C$,
where $C$ is a constant and, in our case, $C=\frac{1}{2}$.
%%%%%%%%%%%%%%%
%%% Its Bounds
%%%%%%%%%%%%%%%
We provide simple upper and lower bounds of $\text{spmax}(z)$ with respect to $\max(z)$
\begin{equation}\label{eqn:spmax_bnd}
\max(z) \le \alpha\text{spmax}\left(\frac{z}{\alpha}\right) \le \max(z) + \alpha\frac{d-1}{2d}.
\end{equation}
The lower bound of \textit{sparsemax} is shown in \cite{martins2016softmax}.
However, we provide another proof of the lower bound and the proof for
the upper bound in Appendix \ref{subsec:upper_lower_bnd_spmax}.
%The proof of upper bound is given as follows.
%\begin{theorem}
%For all $z \in \mathbb{R}^{d}$, $\textnormal{spmax}(z) \le \max(z) + \frac{d-1}{2d}$ holds.
%\end{theorem}
%\begin{proof}
%First, we decompose the summation of (\ref{eqn:spmax_dfn}) into two terms as follows:
%\begin{eqnarray*}
%\begin{aligned}
%\text{spmax}(z) 
%&=\frac{1}{2}\sum_{i = 1}^{K}\left(z_{(i)}^2 - \tau(z)^2\right) + \frac{1}{2}\\
%&\le \frac{1}{2}\sum_{i = 1}^{K}p_{i}^{*}(z)\left(z_{(i)} + \tau(z)\right) + \frac{1}{2}\\
%&=\frac{1}{2}\sum_{i = 1}^{K}p_{i}^{*}(z)z_{(i)} + \frac{1}{2}\sum_{i = 1}^{K}\frac{z_{(i)}}{K} - \frac{1}{2K} + \frac{1}{2}
%\end{aligned}
%\end{eqnarray*}
%where $p_{i}^{*} = z_{(i)} - \tau(z)$ which is the optimal solution of (\ref{eqn:simplex_prob}) and $\sum_{i=1}^{K}p_{i}^{*}(z) = 1$ by definition.
%Now, we use the property that, for every $p$ on $d-1$ dimensional simplex, $\sum_{i}^{d}p_{i}z_{i} \leq \max(z)$ for all $z \in \mathbb{R}^{d}$.
%By using this property, as $p^{*}(z)$ and $\frac{1}{K}\mathbf{1}$ are on the probability simplex, following inequality is induced,
%\begin{eqnarray*}
%\begin{aligned}
%\text{spmax}(z)
%&\leq \frac{1}{2}\sum_{i = 1}^{K}p_{i}^{*}(z)z_{(i)} + \frac{1}{2}\sum_{i = 1}^{K}\frac{z_{(i)}}{K} - \frac{1}{2K} + \frac{1}{2}\\
%&\leq \max(z) - \frac{1}{2K} + \frac{1}{2} \leq \max(z) - \frac{1}{2d} + \frac{1}{2}
%\end{aligned}
%\end{eqnarray*}
%where $d \ge K$ by definition of $K$.
%Therefore, $\textnormal{spmax}(z) \le \max(z) + \frac{d-1}{2d}$ holds.
%\end{proof}

The bounds (\ref{eqn:spmax_bnd}) show that \textit{sparsemax} is a
bounded and smooth approximation of \textit{max} and, from this fact,
(\ref{eqn:sps_bellman}) can be interpreted as an approximation of the
original Bellman equation. 
Using this notation, $V^{sp}_{\pi}$ can be rewritten as,
\begin{equation*}
V_{\pi}^{sp}(s) = \alpha\text{spmax}\left(\frac{Q_{\pi}^{sp}(s,\cdot)}{\alpha}\right).
\end{equation*}
%The bounds between \textit{max} and \textit{sparsemax} derive the bounds of performances of sparsemax policy compared to the original max policy obtained from an original MDP.

%In addition, it can be observed that the structure of sparse Bellman equation is similar to that of soft Bellman equation
%in that both equations have approximated operations of \textit{max}.
%In case of a soft Bellman equation, log-sum-exp function, i.e. $\log\sum_{x}\exp(f(x))$, is also a bounded approximation of \textit{max}.
%In this regards, this interpretation can also be applied to a soft MDP and soft Bellman equation.
%The proof of performance bounds is discussed in Section \ref{sec:PEB}.

\subsection{Supporting Set of Sparse Optimal Policy}
%%%%%%%%%%%%%%
%%% Effects of varying regularization coefficients is controlling supporting set of optimal actions
%%%%%%%%%%%%%%
The supporting set $S(s)$ of a sparse MDP is a set of actions with
nonzero probabilities and the cardinality of $S(s)$ can be controlled
by regularization coefficient $\alpha$, while the supporting set of
a soft MDP is always the same as the entire action space. 
In a sparse MDP, actions assigned with non-zero probability must
satisfy the following inequality:
\begin{equation}\label{eqn:supp_cnd}
\alpha + iQ_{\pi}^{sp}(s,a_{(i)})>\sum_{j=1}^{i}Q_{\pi}^{sp}(s,a_{(j)}),
\end{equation}
where $a_{(i)}$ indicates the action with the $i$th largest action value.
From this inequality, it can be shown that 
$\alpha$ controls the margin between the largest action value and the
others included in the supporting set.
In other words, as $\alpha$ increases, the cardinality of a supporting
set increases since the action values that satisfy (\ref{eqn:supp_cnd}) increase. 
Conversely, as $\alpha$ decreases, the supporting set decreases.
In extreme cases, if $\alpha$ goes zero, only $a_{(1)}$ will be
included in $S(s)$ and if $\alpha$ goes infinity, the entire actions
will be included in $S(s)$. 
On the other hand, in a soft MDP, the supporting set of a softmax
distribution cannot be controlled by the regularization coefficient
$\alpha$ even if the sharpness of the softmax distribution can be adjusted.
This property makes sparse MDPs have an advantage over soft MDPs,
since we can give a zero probability to non-optimal actions by
controlling $\alpha$. 

\subsection{Connection to Tsallis Generalized Entropy}

The notion of the Tsallis entropy was introduced by C. Tsallis as a
general extension of entropy \cite{Tsallis1988possible} and the
Tsallis entropy has been widely used to describe thermodynamic 
systems and molecular motions. 
Surprisingly, the proposed regularization is closely related to a
special case of the Tsallis entropy. 
The Tsallis entropy is defined as follows: 
\begin{equation*}
S_{q,k}(p) = \frac{k}{q-1}\left(1 - \sum_{i}p_{i}^{q}\right),
\end{equation*}
where $p$ is a probability mass function, 
$q$ is a parameter called \textit{entropic-index}, 
and $k$ is a positive real constant.
Note that, if $q\rightarrow 1$ and $k=1$, $S_{1,1}(p)$ is the same as entropy, i.e., $-\sum_{i}p_{i}\log(p_{i})$.
In \cite{ziebart2010MPAs, bloem2014infinite}, it is shown that
$H(\pi)$ is an extension of $S_{1,1}(\pi(\cdot|s))$ since $H(\pi) = \mathbb{E}_{\pi}[S_{(1,1)}(\pi(\cdot|s))]
=\sum_{s,a}-\pi(a|s)\log(\pi(a|s))\rho(s)$. 

We discover the connection between the Tsallis entropy and the
proposed regularization when $q = 2$ and $k = \frac{1}{2}$. 
\begin{theorem}\label{thm:causal_tsallis_entropy}
The proposed policy regularization $W(\pi)$ is an extension of the
Tsallis entropy with parameters $q = 2$ and $k = \frac{1}{2}$ 
to the version of causal entropy, i.e.,
\begin{equation*}
W(\pi) = \mathbb{E}_{\pi}[S_{2,\frac{1}{2}}(\pi(\cdot|s))].
\end{equation*}
\end{theorem}
The proof is provided in Appendix \ref{subsec:csp_ent}
%\begin{proof}
%The proof is simply done by rewriting our regularization as follows:
%\begin{eqnarray*}
%\begin{aligned}
%&W(\pi) \\
%&= \mathbb{E}\left[\sum_{t=0}^{\infty} \gamma^{t}\frac{1}{2}(1-\pi(a_t|s_t))\right]
%= \sum_{s} \frac{1}{2}(1-\sum_{a} \pi(a|s)^{2})\rho(s)\\
%&=\sum_{s}S_{2,\frac{1}{2}}(\pi(\cdot|s))\rho(s) = \mathbb{E}_{\pi}\left[S_{2,\frac{1}{2}}(\pi(\cdot|s))\right].
%\end{aligned}
%\end{eqnarray*}
%\end{proof}

From this theorem, $W(\pi)$ can be interpreted as an extension of
$S_{2,\frac{1}{2}}(p)$ to the case of causally conditioned
distribution, similarly to the causal entropy. 

\section{Sparse Value Iteration}
%%%%%%%%%%%%%%%
%%% Sparse Value Iteration Algorithm
%%%%%%%%%%%%%%%
In this section, we propose an algorithm for solving a causal sparse
Tsallis entropy regularized MDP problem.
Similar to the original MDP and a soft MDP, the sparse version of
value iteration can be induced from the sparse Bellman equation. 
We first define a sparse Bellman operation 
$U^{sp} : \mathbb{R}^{|\mathcal{S}|} \to \mathbb{R}^{|\mathcal{S}|}$: 
for all $s$,
%%%%%%%%%%%%%%
%%% Define Sparse Optimal Bellman Operator
%%%%%%%%%%%%%%
\begin{eqnarray*}
U^{sp}(x)(s) = \alpha\text{spmax}\left(\frac{r(s,\cdot) + \gamma\sum_{s'} x(s')T(s'|s,\cdot)}{\alpha}\right),
\end{eqnarray*}
where $x$ is a vector in $\mathbb{R}^{|\mathcal{S}|}$ and 
$U^{sp}(x)$ is the resulting vector after applying $U^{sp}$ to $x$ and 
$U^{sp}(x)(s)$ is the element for state $s$ in $U^{sp}(x)$.
Then, the sparse value iteration algorithm can be described simply as
\begin{equation*}
x_{i+1} = U^{sp}(x_i),
\end{equation*}
where $i$ is the number of iterations.
%\begin{algorithm}[t]
%\begin{algorithmic}[1]
%\caption{Sparse Value Iteration}\label{alg:sp_value_iter}
%\State{Arbitrary initialization: $x_0\in\mathbb{R}^{|\mathcal{S}|}$}
%\While{Convergence}
%\State $x_{i+1} = U^{sp}(x_i)$
%\EndWhile
%\end{algorithmic}
%\end{algorithm}
In the following section, we show the convergence and the optimality
of the proposed sparse value iteration method.
%In other words, sparse value iteration provide the optimal solution of
%a sparse MDP and we apply this equation to reinforcement learning in two ways.
%Note that the sparse MDP version of policy iteration can be derived similarly. 
  
\subsection{Optimality of Sparse Value Iteration}
%%%%%%%%%%%%%%
%%% What we want to prove?: Optimality of Sparse value iteration
%%%%%%%%%%%%%%
In this section, we prove the convergence and optimality of the sparse
value iteration method. 
We first show that $U^{sp}$ has monotonic and discounting properties
and, by using those properties, we prove that $U^{sp}$ is a contraction.
Then, by the Banach fixed point theorem, repeatedly applying $U^{sp}$
for an arbitrary initial point always converges into the unique fixed
point. 
%%%%%%%%%%%%%%
%%% Monotonicity Property
%%%%%%%%%%%%%%
\begin{lemma}\label{lem:mon}
$U^{sp}$ is monotone: for $x,y \in \mathbb{R}^{|\mathcal{S}|}$, 
if $x \leq y$, then $U^{sp}(x) \leq U^{sp}(y)$, where $\leq$ indicates
an element-wise inequality. 
\end{lemma}
%\begin{proof}
%In \cite{martins2016softmax}, the monotonicity of (\ref{eqn:spmax_dfn}) is proved.
%Then, the monotonicity of $U^{sp}$ can be trivially proved by using the monocity of (\ref{eqn:spmax_dfn}).
%Let $x$ and $y$ are given such that $x \leq y$.
%Then, 
%\begin{eqnarray*}
%\begin{aligned}
%&\alpha\text{spmax}\left(\frac{r(s,a) + \gamma\sum_{s'} x(s')T(s'|s,a)}{\alpha}\right)\\
%&\leq \alpha\text{spmax}\left(\frac{r(s,a) + \gamma\sum_{s'} y(s')T(s'|s,a)}{\alpha}\right)\\
%&\therefore\;\; U^{sp}(x) \leq U^{sp}(y).
%\end{aligned}
%\end{eqnarray*}
%\end{proof}
%%%%%%%%%%%%%%
%%% Discounting Property
%%%%%%%%%%%%%%
\begin{lemma}\label{lem:dis}
For any constant $c\in\mathbb{R}$, $U^{sp}(x+c\mathbf{1}) = U^{sp}(x)+\gamma c\mathbf{1}$, where $\mathbf{1} \in \mathbb{R}^{|\mathcal{S}|}$ 
is a vector of all ones.
\end{lemma}
%\begin{proof}
%In \cite{martins2016softmax}, it is proved that for $c \in \mathbb{R}$ and $x \in \mathbb{R}^{|\mathcal{S}|}$, $\text{spmax}(x+c\mathbf{1}) = \text{spmax}(x) + c$.
%By using this property, for all state $s$,
%\begin{eqnarray*}
%\begin{aligned}
%&U^{sp}(x + c\mathbf{1})(s) \\
%&= \alpha\text{spmax}\left(\frac{r(s,\cdot) + \gamma\sum_{s'}x(s')T(s'|s,\cdot)}{\alpha}\right) + \gamma c\\
%&= U^{sp}(x)(s)+\gamma c.
%\end{aligned}
%\end{eqnarray*}
%\end{proof}

The full proofs can be found in Appendix \ref{subsec:conv_opt_sp_value}.
The proofs of Lemma \ref{lem:mon} and Lemma \ref{lem:dis} rely on the
bounded property of the sparsemax operation. 
It is possible to prove that the sparse Bellman operator $U^{sp}$ is a
contraction using Lemma \ref{lem:mon} and Lemma \ref{lem:dis} 
as follows:
%%%%%%%%%%%%%%
%%% Contraction Property and Banahca Fixed Point Theorem
%%%%%%%%%%%%%%
\begin{lemma}\label{lem:fixed}
$U^{sp}$ is a $\gamma$-contraction mapping and have a unique fixed
point, where $\gamma$ is in $(0,1)$ by definition. 
\end{lemma}
%\begin{proof}
%We first prove that $U^{sp}$ is $\gamma$-contraction mapping with respect to infinite norm $d_{max}$.
%Let $d_{max}(x,y) = M$. Then,$y-M\mathbf{1}\leq x\leq y + M\mathbf{1}$ is satisfied.
%By monotone and discounting properties, following inequality between mapping $U^{sp}(x)$ and $U^{sp}(y)$ is established.
%\begin{eqnarray*}
%\begin{aligned}
%U^{sp}(y)-\gamma M\mathbf{1}\leq U^{sp}(x)\leq U^{sp}(y) + \gamma M\mathbf{1}
%\end{aligned}
%\end{eqnarray*}
%where $\gamma$ is discounting factor of $U^{sp}$.
%From this inequality, $d_{max}(U^{sp}(x),U^{sp}(y)) \leq \gamma M = \gamma d_{max}(x,y)$
%and $\gamma < 1$ since $\gamma$ is a discount factor of an MDP.
%Therefore, $U^{sp}$ is $\gamma$-contraction mapping.
%As $\mathbb{R}^{|\mathcal{S}|}$ with $d_{max}$ is a non-empty complete metric space, 
%by Banach fixed-point theorem, $\gamma$-contraction mapping $U^{sp}$ has a unique fixed point.
%\end{proof}

Using Lemma \ref{lem:mon}, Lemma \ref{lem:dis}, and Lemma \ref{lem:fixed}, 
the optimality and convergence of sparse value iteration can be proven.
%%%%%%%%%%%%%%
%%% Optimality of Sparse Value Iteration
%%%%%%%%%%%%%%
\begin{theorem}\label{thm:opt_sps_mdp}
Sparse value iteration converges to the optimal value of (\ref{eqn:sps_mdp}).
\end{theorem}
%\begin{proof}
%Sparse value iteration converges into the fixed point of $U^{sp}$ by the contraction property.
%Let $x_{*}$ be a fixed point of $U^{sp}$ and, by definition of $U^{sp}$, $x_{*}$ is the point that satisfies sparse Bellman equation, i.e., $x_{*} = U^{sp}(x_{*})$.
%Hence, by Theorem \ref{thm:sp_bellman_eqn}, $x_{*}$ satisfies necessity conditions of optimal solution.
%By the Banach fixed point theorem, $x_{*}$ is a unique point which satisfies necessity conditions of optimal solution.
%In other words, there is no other point that satisfies the sparse Bellman equation.
%Therefore, $x_{*}$ is the optimal value of a sparse MDP.
%\end{proof}

The proof can be found in Appendix \ref{subsec:conv_opt_sp_value}.
Theorem \ref{thm:opt_sps_mdp} is proven using the uniqueness of the
fixed point of $U^{sp}$ and the sparse Bellman equation. 

%% To summarize this section, we have proposed the sparse value iteration
%% method which finds the optimal solution of a sparse MDP problem. 
%% The optimality and convergence of sparse value iteration are
%% guaranteed by Theorem \ref{thm:opt_sps_mdp}. 

\section{Performance Error Bounds for Sparse Value Iteration}\label{sec:PEB}

\begin{algorithm*}[t!]
\small
\caption{Sparse Deep Q-Learning}
\begin{algorithmic}[1] \label{spdqn}
\STATE Initialize prioritized replay memory $M=\emptyset$, Q network parameters $\theta$ and $\theta^{-}$
\FOR {$i = 0$ to $N$}
\STATE Sample initial state $s_0\sim d_{0}(s)$
\FOR {$t = 0$ to $T$}
\STATE Sample action $a_t\sim\pi^{sp}(a|s_t)$ (\ref{eqn:sps_bellman})
\STATE Excute $a_t$ and observe next state $s_{t+1}$ and reward $\mathbf{r}_t$
\STATE Add experiences to replay memory $M$ with an initial importance weight, $M\leftarrow {(s_t, a_t, \mathbf{r}_t, s_{t+1}, w_{0})\cup M}$
\STATE Sample mini-batch $B$ from $M$ based on importance weight
\STATE Set a target value $y_j$ of $(s_j, a_j, \mathbf{r}_j, s_{j+1}, w_{j})$ in $B$, $y_j = \mathbf{r}_j + \gamma \alpha \text{spmax}\left(\frac{Q(s_{j+1},\cdot;\theta^{-})}{\alpha} \right)$
\STATE Minimize $\sum_{j}w_{j}\left(y_j - Q(s_j,a_j;\theta)\right)^2$ using a gradient descent method
\STATE Update importance weights $\{w_{j}\}$ based on temporal difference error $\delta_j = |y_j - Q(s_j,a_j;\theta)|$ \cite{schaul2015prioritized}
\ENDFOR
\STATE Update $\theta^{-} = \theta$ every $c$ iteration
\ENDFOR
\end{algorithmic}
\end{algorithm*}

%%%%%%%%%%%%%%
%%% What we want to show: Sparse Value Iteration has Performance Bounds
%%%%%%%%%%%%%%
We prove the bounds of the performance gap between the policy obtained
by a sparse MDP and the policy obtained by the original MDP, where
this performance error is caused by regularization. 
The boundedness of (\ref{eqn:spmax_bnd}) plays an crucial role to prove the error bounds.
The performance bounds can be derived from bounds of \textit{sparsemax}.
A similar approach can be applied to prove the error bounds of a soft
MDP since a log-sum-exp function is also a bounded approximation of
the \textit{max} operation.
Comparison of log-sum-exp and sparsemax operation is provided in Appendix \ref{subsec:comp_lse_spmax}

Before explaining the performance error bounds,
we introduce two useful propositions which are employed to prove the
performance error bounds of a sparse MDP and a soft MDP.
We first prove an important fact which shows that the optimal values
of sparse value iteration and soft value iteration are greater than
that of the original MDP. 
%The upper error bound of the performance is induced from this fact.

\begin{lemma}\label{lem:operation}
Let $U$ and $U^{soft}$ be the Bellman operations of an original MDP
and soft MDP, respectively, 
such that, for state $s$ and $x \in \mathbb{R}^{|\mathcal{S}|}$, 
\begin{eqnarray*}
\small
\begin{aligned}
&U(x)(s) = \max_{a'}\left(r(s,a') + \gamma \sum_{s'}x(s')T(s'|s,a')\right)\\
&U^{soft}(x)(s) =\alpha\log\sum_{a'}\exp\left(\frac{r(s,a') + \gamma \sum_{s'}x(s')T(s'|s,a')}{\alpha}\right).
\end{aligned}
\end{eqnarray*}
Then following inequalities hold for every integer $n>0$:
{\small
\begin{equation*}
U^{n}(x) \leq (U^{sp})^{n}(x),\;\;
U^{n}(x) \leq (U^{soft})^{n}(x),
\end{equation*}}
where $U^{n}$ (resp., $(U^{sp})^{n}$) is the result after applying $U$
(resp., $U^{sp}$) $n$ times.
In addition, let $x_{*},\;x_{*}^{sp}$ and $x_{*}^{soft}$ be the fixed
points of $U,\;U^{sp}$ and $U^{soft}$, respectively. 
Then, following inequalities also hold:
\begin{equation*}
x_{*} \leq x_{*}^{sp},\;\;
x_{*} \leq x_{*}^{soft}.
\end{equation*}
\end{lemma}

%\begin{proof}
%We first prove inequality about a sparse Bellman operation
%\begin{equation*}
%U^{n}(x) \leq (U^{sp})^{n}(x),\;\;x_{*} \leq x_{*}^{sp}.
%\end{equation*}
%This inequality can be proven by the mathematical induction.
%When $n=1$, the inequality is proven by using $\max(z) \leq \text{spmax}(z)$ as follows:
%\begin{eqnarray*}
%&\max_{a'}\left(r(s,a') + \gamma \sum_{s'}x(s')T(s'|s,a')\right) \\
%&\leq \text{spmax}\left(r(s,\cdot) + \gamma \sum_{s'}x(s')T(s'|s,\cdot)\right)\\
%&\therefore U(x) \leq U^{sp}(x).
%\end{eqnarray*}
%Let us assume that, for every $x \in \mathbb{R}^{|\mathcal{S}|}$, $U^{k}(x) \leq (U^{sp})^{k}(x)$ for some $k$.
%Then, when $n = k+1$, 
%\begin{eqnarray*}
%&U^{k+1}(x) = U^{k}(U(x)) \leq (U^{sp})^{k}(U(x))\\
%&\leq (U^{sp})^{k}(U^{sp}(x)).
%\end{eqnarray*}
%Therefore, by mathematical induction, it is satisfied $U^{n}(x) \leq (U^{sp})^{n}(x)$ for every integer $n$.
%Then, the inequality of the fixed points of $U$ and $U^{sp}$ can be obtained by $n\rightarrow \infty$,
%\begin{equation*}
%x_{*} \leq x^{sp}_{*}
%\end{equation*}
%where $*$ indicates an optimal solution.
%The above arguments also hold when $U^{sp}$ and \textit{sparsemax} are replaced with $U^{soft}$ and \textit{log-sum-exp} operation, respectively.
%Then, $U^{n}(x) \leq (U^{soft})^{n}(x)$ and $x_{*} \leq x_{*}^{soft}$ are also proved by the mathematical induction.
%\end{proof}
The detailed proof is provided in Appendix \ref{sec:error_bnd}.
Lemma \ref{lem:operation} shows that the optimal values,
$V^{sp}_{\pi}$ and $V^{soft}_{\pi}$, obtained by sparse value
iteration and soft value iteration are always greater than the
original optimal value $V_{\pi}$. %obtained by the conventional value iteration. 
Intuitively speaking, the reason for this inequality is due to the
regularization term, i.e., $W(\pi)$ or $H(\pi)$, added to the
objective function.

Now, we discuss other useful properties about the proposed causal
sparse Tsallis entropy regularization $W(\pi)$ and causal entropy
regularization $H(\pi)$. 

\begin{lemma} \label{lem:reg_bnd}
$W(\pi)$ and $H(\pi)$ have following upper bounds:
%{\small
\begin{equation*}
W(\pi) \leq \frac{1}{1-\gamma}\frac{|\mathcal{A}|-1}{2|\mathcal{A}|},\;\;
H(\pi) \leq \frac{\log(|\mathcal{A}|)}{1-\gamma}
\end{equation*}%}
where $|\mathcal{A}|$ is the cardinality of the action space $\mathcal{A}$.
\end{lemma}
%\begin{proof}
%For $W(\pi)$,
%\begin{eqnarray*}
%W(\pi) &=& \mathbb{E}\left[\sum_{t=0}^{\infty}\gamma^{t}\frac{1}{2}(1 - \pi(a_t|s_t))\middle|\pi\right]\\
%&=& \sum_{s}\rho_{\pi}(s)\sum_{a}\frac{1}{2}(1 - \pi(a|s))\pi(a|s)\\
%&\leq& \sum_{s}\rho_{\pi}(s)\frac{|\mathcal{A}|-1}{2|\mathcal{A}|}= \frac{1}{1-\gamma}\frac{|\mathcal{A}|-1}{2|\mathcal{A}|}.
%\end{eqnarray*}
%Similarly, for $H(\pi)$,
%\begin{eqnarray*}
%H(\pi) &=& \mathbb{E}\left[\sum_{t=0}^{\infty}\gamma^{t}-\log(\pi(a_t|s_t))\middle|\pi\right]\\
%&=& \sum_{s}\rho_{\pi}(s)\sum_{a}-\log(\pi(a|s))\pi(a|s)\\
%&\leq& \sum_{s}\rho_{\pi}(s)\log(|\mathcal{A}|)=\frac{1}{1-\gamma}\log(|\mathcal{A}|).
%\end{eqnarray*}
%\end{proof}
The proof is provided in Appendix \ref{sec:error_bnd}.
Theorem \ref{lem:reg_bnd} can be induced by extending the upper bound of $S_{1,1}(\pi)$ and $S_{2,\frac{1}{2}}(\pi)$ to the causal entropy and causal sparse Tsallis entropy.

By using Lemma \ref{lem:operation} and Lemma \ref{lem:reg_bnd},
the performance bounds for a sparse MDP and a soft MDP can be derived
as follows. 
%%%%%%%%%%%%%%
%%% Bounds for Sparse Value Iteration
%%%%%%%%%%%%%%
\begin{theorem}\label{thm:bnd_sps_mdp}
Following inequalities hold:
{\small
\begin{equation*}
\mathbb{E}_{\pi^{*}}(\mathbf{r}(s,a)) - \frac{\alpha}{1-\gamma}\frac{|\mathcal{A}|-1}{2|\mathcal{A}|} \leq \mathbb{E}_{\pi^{sp}}(\mathbf{r}(s,a)) \leq \mathbb{E}_{\pi^{*}}(\mathbf{r}(s,a)), 
\end{equation*}
}
where $\pi^{*}$ and $\pi^{sp}$ are the optimal policy obtained by the
original MDP and a sparse MDP, respectively. 
%and $|\mathcal{A}|$ is the cardinality of the action space. 
\end{theorem}
%\begin{proof}
%Let $\pi^{*}$ be an optimal policy of an original MDP (\ref{eqn:bellman}).
%Then, by definition of optimality, right-side inequality is trivial.
%To prove the left-side inequality, following inequality is utilized, which is induced from Lemma \ref{lem:operation}:
%\begin{equation}\label{eqn:sps_value_ieq}
%V_{\pi^{*}} \leq V^{sp}_{\pi^{sp}}.
%\end{equation}
%Since $\pi^{*}$ and $\pi^{sp}$ is the optimal solution of an original MDP and sparse MDP,
%the value of $\pi^{*}$ and $\pi^{sp}$ is the optimal value of each MDP and above inequality is true.
%The proof of left-side inequality is done by using (\ref{eqn:sps_value_ieq}).
%\begin{eqnarray*}
%\mathbb{E}_{\pi^{*}}(\mathbf{r}(s,a)) &=& \sum_{s'}d(s')V_{\pi^{*}}(s') \leq \sum_{s'}d(s')V^{sp}_{\pi^{sp}}(s') \\
%&=& J^{sp}_{\pi^{sp}} = \mathbb{E}_{\pi^{sp}}(\mathbf{r}(s,a)) + \alpha W(\pi^{sp})\\
%&\leq& \mathbb{E}_{\pi^{sp}}(\mathbf{r}(s,a)) + \frac{\alpha}{1-\gamma}\frac{|\mathcal{A}|-1}{2|\mathcal{A}|}.
%\end{eqnarray*}
%\end{proof}
%%%%%%%%%%%%%%
%%% Bounds for Soft Max and Soft Value Iteration and Comparison Results
%%%%%%%%%%%%%%
\begin{theorem}\label{thm:bnd_soft_mdp}
Following inequalities hold:
{\small
\begin{equation*}
\mathbb{E}_{\pi^{*}}(\mathbf{r}(s,a)) - \frac{\alpha}{1-\gamma}\log(|\mathcal{A}|) \leq \mathbb{E}_{\pi^{soft}}(\mathbf{r}(s,a)) \leq \mathbb{E}_{\pi^{*}}(\mathbf{r}(s,a)) 
\end{equation*}
}
where $\pi^{*}$ and $\pi^{soft}$ are the optimal policy obtained by
the original MDP and a soft MDP, respectively.
% and $|\mathcal{A}|$ is the cardinality of the action space. 
\end{theorem}
%\begin{proof}
%The proof scheme of this theorem is the same as Lemma \ref{thm:bnd_sps_mdp}
%by replacing $W(\pi)$ with $H(\pi)$.
%Detail proof is omitted here and can be found in appendix.
%\end{proof}

The proofs of Theorem \ref{thm:bnd_sps_mdp} and Theorem \ref{thm:bnd_soft_mdp} can be found in Appendix \ref{sec:error_bnd}.
These error bounds show us that the expected return of the optimal
policy of a sparse MDP has always tighter error bounds than that of a
soft MDP. 
Moreover, it can be also known that the bounds for the proposed sparse MDP converges to a constant $\frac{\alpha}{2(1-\gamma)}$ as the number of actions increases, 
whereas the error bounds of soft MDP grows logarithmically.

This property has a clear benefit when a sparse MDP is applied to a
robotic problem with a continuous action space.
To apply an MDP to a continuous action space,
a discretization of the action space is essential and 
a fine discretization is required to obtain a solution which
is closer to the underlying continuous optimal policy. 
Accordingly, the number of actions becomes larger as the level of
discretization increases. 
In this case, a sparse MDP has advantages over a soft MDP in that the
performance error of a sparse MDP is bounded by a constant factor as
the number of actions increases, whereas performance error of optimal
policy of a soft MDP grows logarithmically. 
%To verify this theory, we experimented with measuring the performance of the expected return while increasing the number of actions $|\mathcal{A}|$ as shown in [Figure].
%We create transition model $T$ by discretization of unicycle dynamics defined in continuous state and action space and solve an original MDP, soft MDP and sparse MDP under predefined rewards while increasing discrete level of action space.
%As shown in [Figure], the performance of sparse MDP converges to the constant bound while the performance of the soft MDP grows with log scale.

%%%%%%%%%%%%%%
%%% Cardinality of Optimal Policy Distribution 
%%%%%%%%%%%%%%
\section{Sparse Exploration and Update Rule for Sparse Deep Q-Learning}

In this section, we first propose sparse Q-learning and further extend
to sparse deep Q-learning where a sparsemax policy and the sparse Bellman
equation are employed as a exploration method and update rule. 

Sparse Q-learning is a model free method to solve the proposed sparse
MDP without the knowledge of transition probabilities. 
In other words, when the transition probability $T(s'|a,s)$ is unknown
but sampling from $T(s'|a,s)$ is possible, sparse Q-learning
estimates an optimal $Q^{sp}$ of the sparse MDP using sampling, as
Q-learning finds an approximated value of an optimal $Q$ of the
conventional MDP.  
Similar to Q-learning, the update equation of sparse Q-learning is
derived from the sparse Bellman equation, 
\begin{eqnarray*}
%\small
\begin{aligned}
&Q^{sp}(s_i,a_i) \leftarrow Q^{sp}(s_i,a_i) +\\
&\eta(i)\left[\mathbf{r}(s_{i},a_{i}) + \gamma \alpha \text{spmax}\left(\frac{Q^{sp}(s_{i+1},\cdot)}{\alpha} \right)  - Q(s_i,a_i)\right],
\end{aligned}
\end{eqnarray*}
where $i$ indicates the number of iterations and $\eta(i)$ is a learning rate.
If the learning rate $\eta(i)$ satisfies $\sum^{\infty}_{i=0}\eta(i) =
\infty$ and $\sum^{\infty}_{i=0}\eta(i)^{2} < \infty$, 
then, as the number of samples increases to infinity, sparse
Q-learning converges to the optimal solution of a sparse MDP. 
The proof of the convergence and optimality of sparse Q-learning is
the same as that of the standard Q-learning \cite{watkins1992q}. 

The proposed sparse Q-learning can be easily extended to sparse deep Q-learning
using a deep neural network as an estimator of the sparse Q value.
In each iteration, sparse deep Q-learning performs a gradient descent
step to minimize the squared loss 
$(y  -  Q(s,a;\theta))^2$, where $\theta$ is the parameter of the Q network.
Here, $y$ is the target value defined as follows:
\begin{eqnarray*}
%\small
\begin{aligned}
y = \mathbf{r}(s,a) + \gamma \alpha \text{spmax}\left(\frac{Q(s',\cdot;\theta)}{\alpha} \right),
\end{aligned}
\end{eqnarray*}
where $s'$ is the next state sampled by taking action $a$ at the state $s$
and $\theta$ indicates network parameters.

Moreover, we employ the sparsemax policy as the exploration strategy
where the policy distribution is computed by (\ref{eqn:sps_bellman})
with action values estimated by a deep Q network.  
The sparsemax policy excludes the action whose estimated action
value is too low to be re-explored, by assigning zero probability
mass. 
The effectiveness of the sparsemax exploration is investigated in
Section \ref{exp}.

For stable convergence of a Q network, we utilize double Q-learning
\cite{van2016deep}, where the parameter $\theta$ for obtaining a
policy and the parameter $\theta^{-}$ for computing the target value are separated
and $\theta^{-}$ is updated to $\theta$ at every predetermined iterations.
In other words, double Q-learning prevents instability of deep
Q-learning by slowly updating the target value. 
Prioritized experience replay \cite{schaul2015prioritized} is also applied 
where the optimization of a network proceeds in consideration of the
importance of experience. 
The whole process of sparse deep Q-learning is summarized in Algorithm \ref{spdqn}.

\begin{figure}[t]
\centering
\subfigure	[Performance Bounds]{\label{exp:performance_bnd}
  \centering
  \includegraphics[width=.23\textwidth]{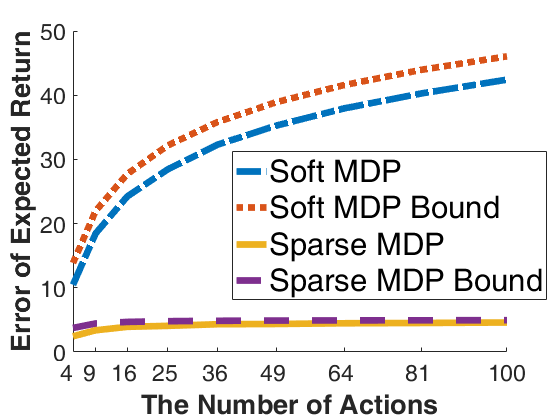}}
\subfigure[Supporting Set Comparison]{\label{exp:supp_set}
  \centering
  \includegraphics[width=.23\textwidth]{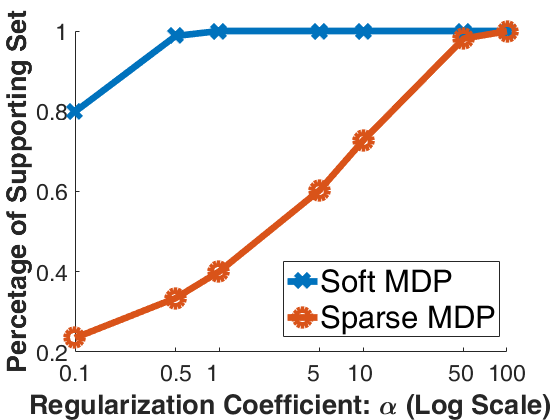}}
\caption{
%\small
(a) The performance gap is calculated as the absolute value of the difference between the performance of sparse MDP or soft MDP and the performance of an original MDP.
(b) The ratio of the number of supporting actions to the total number of actions is shown.
The action space of unicycle dynamics is discretized into $25$ actions.
}
\end{figure}

\section{Experiments}\label{exp}
%%%%%%%%%%%%%%%%
%%% Sparse Policy -> experts
%%%%%%%%%%%%%%%%%
We first verify Theorem \ref{thm:bnd_sps_mdp}, Theorem \ref{thm:bnd_soft_mdp} 
and the effect of (\ref{eqn:supp_cnd}) in simulation.  
For verification of Theorem \ref{thm:bnd_sps_mdp} and Theorem
\ref{thm:bnd_soft_mdp}, we measure the performance of the expected
return while increasing the number of actions, $|\mathcal{A}|$. 
For verification of the effect of (\ref{eqn:supp_cnd}), the
cardinality of the supporting set of optimal policies of sparse and
soft MDP are compared at different values of $\alpha$. 

To investigate effectiveness of the proposed method, 
we test sparsemax exploration and the sparse Bellman update rule on
reinforcement learning with a continuous action space. 
To apply Q-learning to a continuous action space, a fine
discretization is necessary to obtain a solution which is closer to 
the original continuous optimal policy. 
As the level of discretization increases, the number of actions to be
explored becomes larger. 
In this regards, an efficient exploration method is required to obtain
high performance. % when a fine grained discretization of the action space is given.
We compare our method to other exploration methods with respect to the 
convergence speed and the expected sum of rewards. 
We further check the effect of the update rule.

\subsection{Experiments on Performance Bounds and Supporting Set}

To verify our theorem about performance error bounds, we create a
transition model $T$ by discretization of unicycle dynamics defined in
a continuous state and action space and solve the original MDP, a soft
MDP and a sparse MDP under predefined rewards while increasing the
discretization level of the action space.
The reward function is defined as a linear combination of two squared exponential functions, i.e., $\mathbf{r}(x) = \exp\left(\frac{||x-x_{1}||^2}{2\sigma_{1}^{2}}\right)-\exp\left(\frac{||x-x_{2}||^2}{2\sigma_{2}^{2}}\right)$,
where $x$ is a location of a unicycle, $x_1$ is a goal point, $x_2$ is
the point to avoid, and $\sigma_1$ and $\sigma_2$ are scale
parameters. 
The reward function is designed to let an agent to navigate towards
$x_1$ while avoiding $x_2$.
The absolute value of differences between the expected return of the
original MDP and that of sparse MDP (or soft MDP) is measured. 
As shown in Figure \ref{exp:performance_bnd}, the performance gap of
sparse MDP converges to a constant bound while the performance of the
soft MDP grows logarithmically. 
Note that the performance gaps of the sparse MDP and soft MDP are
always smaller than their error bounds. 
Supporting set experiments are conducted using discretized unicycle dynamics. 
The cardinality of optimal policies are measured while $\alpha$ varies
from $0.1$ to $100$. 
In Figure \ref{exp:supp_set}, while the ratio of the supporting set
for a soft MDP is changed from $0.79$ to $1.00$, the ratio for a
sparse MDP is changed from $0.24$ to $0.99$, demonstrating the
sparseness of the proposed sparse MDPs compared to soft MDPs.

\subsection{Reinforcement Learning in a Continuous Action Space}

We test our method in MuJoCo \cite{todorov2012mujoco}, a physics-based
simulator, using two problems with a continuous action space:
\textit{Inverted Pendulum} and \textit{Reacher}. 
The action space is discretized to apply Q-learning to a continuous
action space and experiments are conducted with four different
discretization levels to validate the effectiveness of sparsemax
exploration and the sparse Bellman update rule.

We compare the sparsemax exploration method to the $\epsilon$-greedy
method and softmax exploration \cite{vamplew2017softmax} 
and further compare the sparse Bellman update rule to the original
Bellman update rule \cite{watkins1992q} and the soft Bellman update
rule \cite{bloem2014infinite}. 
In addition, three different regularization coefficient settings are
experimented. 
In total, we test $27$ combinations of variants of deep Q-learning by
combining three exploration methods, three update rules, and three
different regularization coefficients of $0.01, 0.1$, and $1$.
The deep deterministic policy gradient (DDPG) method
\cite{lillicrap2015continuous}, which operates in a continuous action
space without discretization of the action space, is also compared\footnote{
To test DDPG, we used the code from Open AI available at \url{https://github.com/openai/baselines}.}.
Hence, a total of $28$ algorithms are tested.

Results are shown in Figure \ref{exp:invert} and Figure
\ref{exp:reacher} for inverted pendulum and reacher, respectively,
where only the top five algorithms are plotted and each
point in a graph is obtained by averaging the values from three
independent runs with different random seeds.
Results of all $28$ algorithms are provided in Appendix \ref{sec:exp_full}.
Q network with two 512 dimensional hidden layers is used for the inverted
pendulum problem and a network with four 256 dimensional hidden layers
is used for the reacher problem. 
Each Q-learning algorithm utilizes the same network topology.
For inverted pendulum, since the problem is easier than the reacher
problem, most of top five algorithms converge to the maximum return
of $1000$ at each discretization level as shown in Figure
\ref{exp:perf_invert}. 
Four of top five algorithms utilize the proposed sparsemax exploration.
Only one of the top five methods utilizes the softmax
exploration. 
In Figure \ref{exp:epi_invert}, the number of episodes required to
reach a near optimal return, 980, is shown.
The sparsemax exploration requires a less number of episodes to
obtain a near optimal value than $\epsilon$-greedy, softmax exploration. 

For the reacher problem, the algorithms with sparsemax exploration
slightly outperforms $\epsilon$-greedy methods and the performance of
softmax exploration is not included in the top five as shown in
Figure \ref{exp:perf_reacher}.  
In terms of the number of required episodes, sparsemax exploration
outperforms epsilon greedy methods as shown in Figure
\ref{exp:epi_reacher}, where we set the threshold return to be $-6$. 
DDPG shows poor performances in both problems since the number of
sampled episodes is insufficient. 
In this regards, deep Q-learning with sparsemax exploration
outperforms DDPG with less number of episodes. 
From these experiments, it can be known that the sparsemax
exploration method has an advantage over softmax exploration,
$\epsilon$-greedy method and DDPG with respect to the number of
episodes required to reach the optimal performance. 

\begin{figure}[t]
\subfigure[Expected Return]{\label{exp:perf_invert}
  \centering
  \includegraphics[width=.23\textwidth]{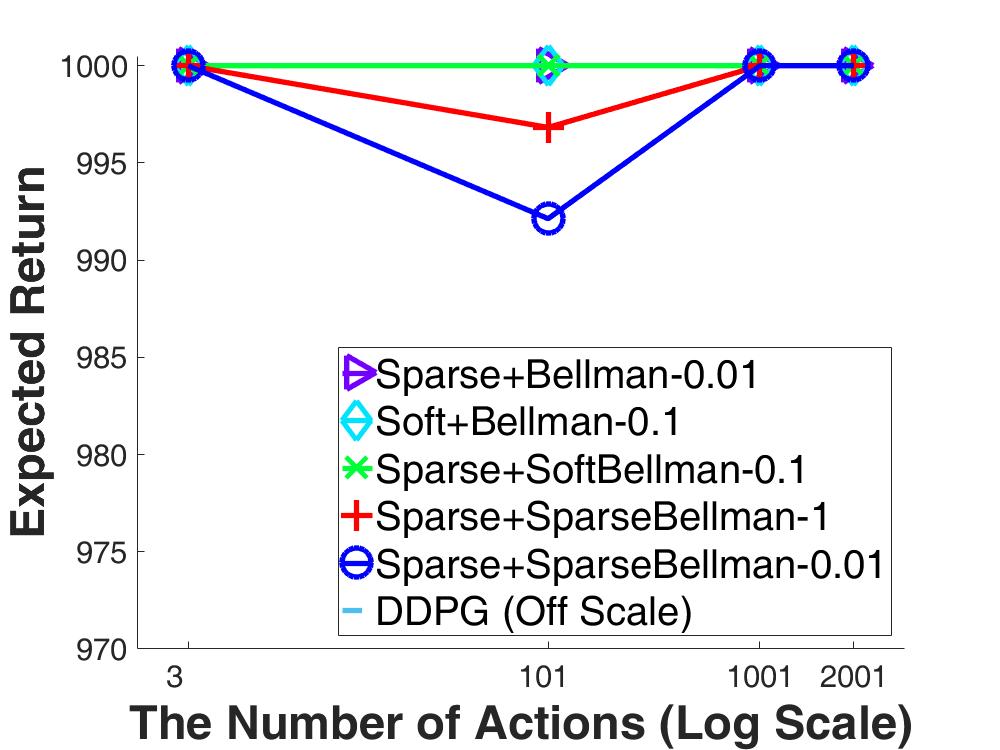}}
\subfigure	[Required Episodes]{\label{exp:epi_invert}
  \centering
  \includegraphics[width=.23\textwidth]{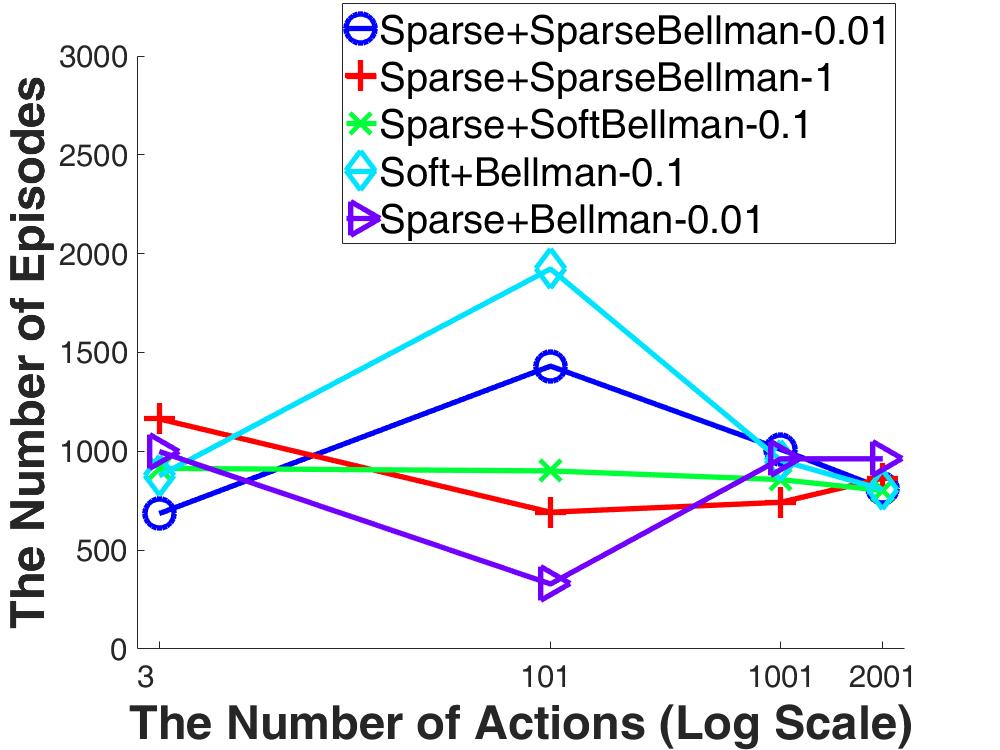}}
\caption{
\small
Inverted pendulum problem.
Algorithms are named as \texttt{<exploration method>+<update rule>+<$\alpha$>}.
(a) The average performance of each algorithm after $3000$ episodes.
The performance of DDPG is out of scale.
(b) The average number of episodes required to reach the threshold
value $980$.}
\label{exp:invert}
\end{figure}

\begin{figure}[t]
\subfigure[Expected Return]{\label{exp:perf_reacher}
  \centering
  \includegraphics[width=.23\textwidth]{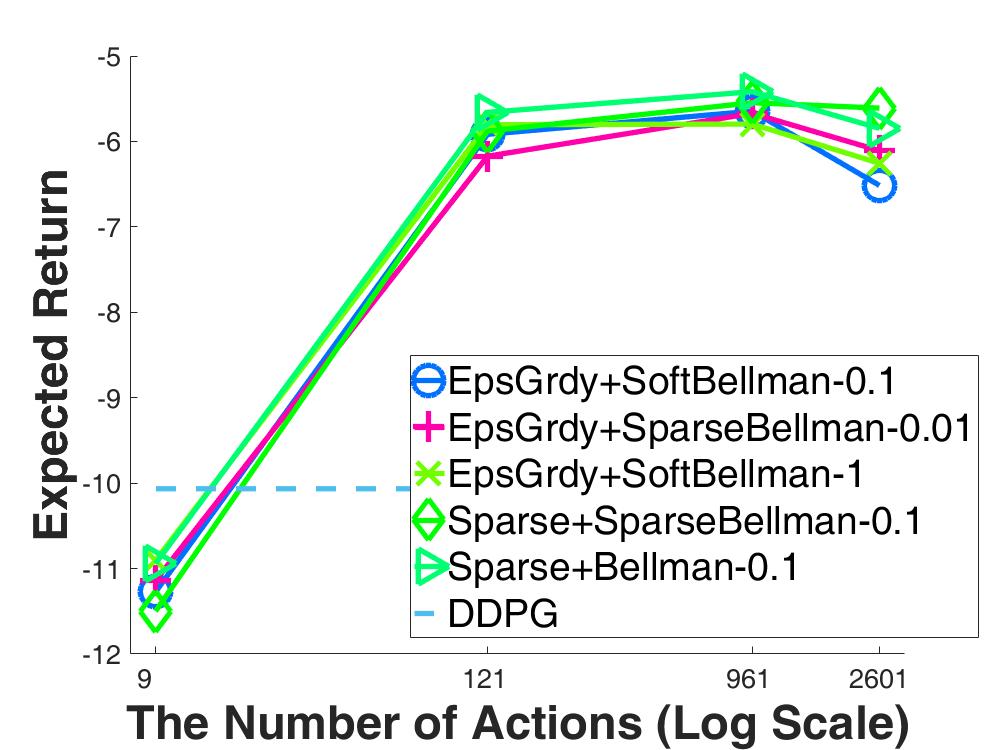}}
\subfigure	[Required Episodes]{\label{exp:epi_reacher}
  \centering
  \includegraphics[width=.23\textwidth]{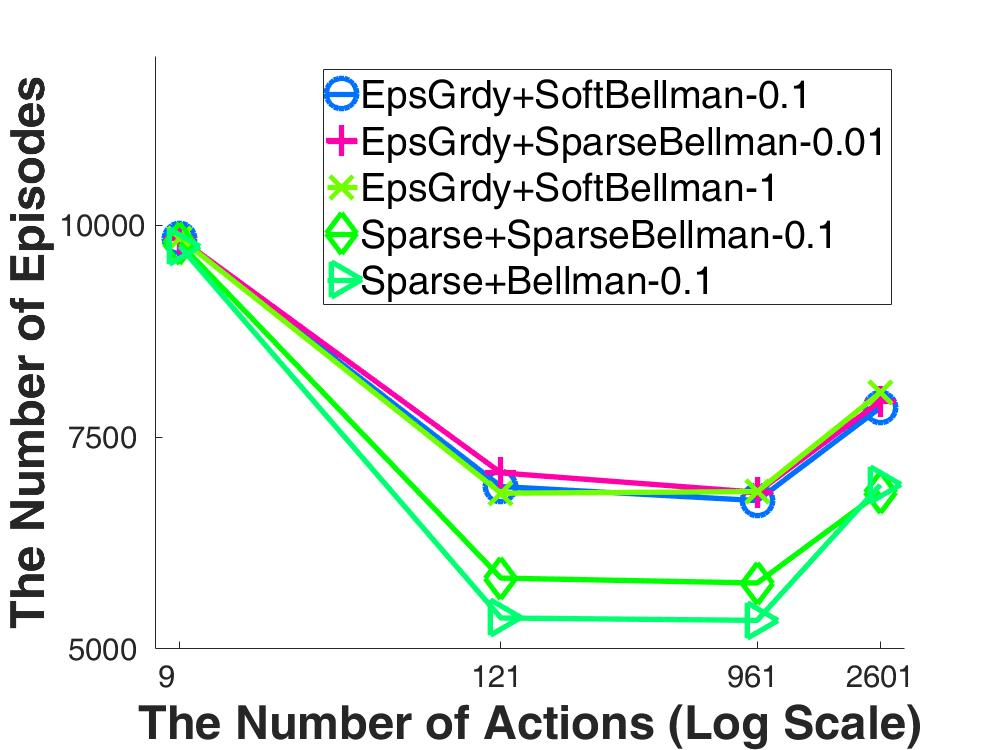}}
\caption{
\small
Reacher problem.
(a) The average performance of each algorithm after $10000$ episodes.
(b) The average number of episodes required to reach the threshold value $-6$.}
\label{exp:reacher}
\end{figure}

\section{Conclusion}

In this paper, we have proposed a new MDP with novel causal sparse
Tsallis entropy regularization which induces a sparse and multi-modal
optimal policy distribution. 
In addition, we have provided the full mathematical analysis of the
proposed sparse MDPs:
the optimality condition of sparse MDPs given as the sparse Bellman equation,
sparse value iteration and its convergence and optimality properties, and
the performance bounds between the propose MDP and the original MDP.
We have also proven that the performance gap of a sparse MDP is
strictly smaller than that of a soft MDP. 
In experiments, we have verified that the theoretical performance gaps of a
sparse MDP and soft MDP from the original MDP are correct. 
We have applied the sparsemax policy and sparse Bellman equation to
deep Q-learning as the exploration strategy and update rule,
respectively, and shown that the proposed exploration method shows
significantly better performance compared to $\epsilon$-greedy,
softmax exploration, and DDPG, as the number of actions increases.
From the analysis and experiments, we have demonstrated that the
proposed sparse MDP can be an efficient alternative to problems with a
large number of possible actions and even a continuous action space.

\appendices

\section{}\label{sec:analysis}

\subsection{Sparse Bellman Equation from Karush-Kuhn-Tucker conditions}\label{subsec:sp_bellman_eq}

The following proof explains the optimality condition of the sparse MDP from Karush-Kuhn-Tucker (KKT) conditions.

\begin{proof}[Proof of Theorem \ref{thm:sp_bellman_eqn}]
The KKT conditions of (\ref{eqn:sps_mdp}) are as follows:
\begin{eqnarray}\label{eqn:std_mdp_kkt1}
&\forall s, a &\sum_{a'}\pi(a'|s) - 1 = 0,\;\; -\pi(a|s) \leq 0 \\ \label{eqn:std_mdp_kkt2}
&\forall s,a &\lambda_{sa} \ge 0\\  \label{eqn:std_mdp_kkt3}
&\forall s,a &\lambda_{sa}\pi(a|s) = 0\\  \label{eqn:std_mdp_kkt4}
&\forall s,a &\frac{\partial L(\pi,c,\lambda)}{\partial \pi(a|s)} = 0  
\end{eqnarray}
where $c$ and $\lambda$  are Lagrangian multipliers for the equality and inequality constraints, respectively, and (\ref{eqn:std_mdp_kkt1}) is the feasibility of primal variables, (\ref{eqn:std_mdp_kkt2}) is the feasibility of dual variables, (\ref{eqn:std_mdp_kkt3}) is the complementary slackness and (\ref{eqn:std_mdp_kkt4}) is the stationarity condition.
The Lagrangian function of (\ref{eqn:sps_mdp}) is written as follows:
\begin{equation*}
\begin{aligned}
&L(\pi,c,\lambda) \\
&= -J^{sp}_{\pi} + \sum_{s} c_{s} \left(\sum_{a'}\pi(a'|s) - 1\right) - \sum_{s,a}\lambda_{sa}\pi(a|s)
\end{aligned}
\end{equation*}
where the maximization of (\ref{eqn:sps_mdp}) is changed into the minimization problem, i.e., $\min_{\pi} -J^{sp}_{\pi}$.
First, the derivative of $J^{sp}_{\pi}$ can be obtained by using the chain rule.
\begin{eqnarray*}
\begin{aligned}
&\frac{\partial J_{\pi}}{\partial \pi(a|s)} =
d^{\intercal}G_{\pi}^{-1}\frac{\partial r^{sp}_{\pi}}{\partial \pi(a|s)} + \gamma d^{\intercal}G_{\pi}^{-1}\frac{\partial T_{\pi}}{\partial \pi(a|s)}G_{\pi}^{-1}r^{sp}_{\pi}\\
&= \rho_{\pi}^{\intercal}\frac{\partial r^{sp}_{\pi}}{\partial \pi(a|s)}+\gamma\rho_{\pi}^{\intercal}\frac{\partial T_{\pi}}{\partial \pi(a|s)}V^{sp}_{\pi}\\
&= \rho_{\pi}(s)\left(r(s,a) + \frac{\alpha}{2} - \alpha\pi(a|s) + \gamma\sum_{s'}V^{sp}_{\pi}(s')T(s'|s,a)\right) \\
&= \rho_{\pi}(s)\left(Q_{\pi}^{sp}(s,a) + \frac{\alpha}{2} - \alpha\pi(a|s)\right).
\end{aligned}
\end{eqnarray*}
Here, the partial derivative of Lagrangian is obtained as follows:
\begin{eqnarray*}
\begin{aligned}
&\frac{\partial L(\pi,c,\lambda)}{\partial \pi(a|s)} \\
&= -\rho_{\pi}(s)(Q_{\pi}^{sp}(s,a) + \frac{\alpha}{2} - \alpha\pi(a|s)) + c_{s} - \lambda_{sa} = 0.
\end{aligned}
\end{eqnarray*}
First, consider a positive $\pi(a|s)$ where the corresponding Lagrangian multiplier $\lambda_{sa}$ is zero due to the complementary slackness.
By summing $\pi(a|s)$ with respect to action $a$, Lagrangian multiplier $c_{s}$ can be obtained as follows:
\begin{eqnarray*}
\small
\begin{aligned}
&0=-\rho_{\pi}(s)(Q_{\pi}^{sp}(s,a) + \frac{\alpha}{2} - \alpha\pi(a|s)) + c_{s}\\
&\pi(a|s) = \left(-\frac{c_{s}}{\rho_{\pi}(s)\alpha}+\frac{1}{2}+\frac{Q_{\pi}^{sp}(s,a)}{\alpha}\right)\\
&\sum_{\pi(a'|s) > 0}\pi(a'|s) = \sum_{\pi(a'|s) > 0}\left(-\frac{c_{s}}{\rho_{\pi}(s)\alpha}+\frac{1}{2}+\frac{Q_{\pi}^{sp}(s,a')}{\alpha}\right)=1\\
&\therefore c_{s} = \rho_{\pi}(s)\alpha\left[\frac{\sum_{\pi(a'|s)>0}\frac{Q_{\pi}^{sp}(s,a')}{\alpha} - 1}{K} + \frac{1}{2}\right]
\end{aligned}
\end{eqnarray*}
where $K$ is the number of positive elements of $\pi(\cdot|s)$.
By replacing $c_s$ with this result, the optimal policy distribution is induced as follows.
\begin{eqnarray*}
\begin{aligned}
\pi(a|s) &= \left(-\frac{c_{s}}{\rho_{\pi}(s)\alpha}+\frac{1}{2}+\frac{Q_{\pi}^{sp}(s,a)}{\alpha}\right)\\
&= \frac{Q_{\pi}^{sp}(s,a)}{\alpha} - \frac{\sum_{\pi(a'|s)>0}\frac{Q_{\pi}^{sp}(s,a')}{\alpha} - 1}{K}
\end{aligned}
\end{eqnarray*}
As this equation is derived under the assumption that $\pi(a|s)$ is positive.
For $\pi(a|s) > 0$, following condition is necessarily fulfilled,
\begin{eqnarray*}
\begin{aligned}
\frac{Q_{\pi}^{sp}(s,a)}{\alpha} > \frac{\sum_{\pi(a'|s)>0}\frac{Q_{\pi}^{sp}(s,a')}{\alpha} - 1}{K}.
\end{aligned}
\end{eqnarray*}
We notate this supporting set as $S(s) = \{a|1+K\frac{Q_{\pi}^{sp}(s,a)}{\alpha} > \sum_{\pi(a'|s)>0}\frac{Q_{\pi}^{sp}(s,a')}{\alpha}\}$.
$S(s)$ contains the actions which has larger action values than threshold 
\begin{equation*}
\tau(Q_{\pi}^{sp}(s,\cdot)) = \frac{\sum_{\pi(a'|s)>0}\frac{Q_{\pi}^{sp}(s,a')}{\alpha} - 1}{K}.
\end{equation*}
By using these notations, the optimal policy distribution can be rewritten as follows:
\begin{eqnarray*}
\begin{aligned}
\pi(a|s) = \max\left(\frac{Q_{\pi}^{sp}(s,a)}{\alpha} -\tau\left(\frac{Q_{\pi}^{sp}(s,\cdot)}{\alpha}\right),0\right).
\end{aligned}
\end{eqnarray*}
By substituting $\pi(a|s)$ with this result, the following optimality equation of $V^{sp}_{\pi}$ is induced.
\begin{eqnarray*}
\small
\begin{aligned}
&V^{sp}_{\pi}(s) \\
&= \sum_{a}\pi(a|s)\left(Q_{\pi}^{sp}(s,a) + \frac{\alpha}{2}(1-\pi(a|s))\right)\\
&= \sum_{a}\pi(a|s)\left(Q_{\pi}^{sp}(s,a) -\frac{\alpha}{2}\pi(a|s)\right) + \frac{\alpha}{2}\sum_{a}\pi(a|s)\\
&= \sum_{a\in S(s)}\pi(a|s)\\
&\times \left(Q_{\pi}^{sp}(s,a) -\frac{\alpha}{2}\left(\frac{Q_{\pi}^{sp}(s,a)}{\alpha} -\tau\left(\frac{Q_{\pi}^{sp}(s,\cdot)}{\alpha}\right)\right)\right) + \frac{\alpha}{2}\\
&= \sum_{a\in S(s)}\pi(a|s)\frac{\alpha}{2}\left(\frac{Q_{\pi}^{sp}(s,a)}{\alpha} + \tau\left(\frac{Q_{\pi}^{sp}(s,\cdot)}{\alpha}\right)\right) + \frac{\alpha}{2}\\
&=  \alpha\left[\frac{1}{2}\sum_{a \in S(s)}^{K}\left(\left(\frac{Q_{\pi}^{sp}(s,a)}{\alpha}\right)^2 - \tau\left(\frac{Q_{\pi}^{sp}(s,\cdot)}{\alpha}\right)^2\right) + \frac{1}{2}\right]
\end{aligned}
\end{eqnarray*}

To summarize, we obtain the sparse Bellman equation as follows:
\begin{eqnarray*}
\footnotesize
\begin{aligned}
Q_{\pi}^{sp}(s,a) &= r(s,a) + \gamma\sum_{s'} V^{sp}_{\pi}(s')T(s'|s,a)\\
V^{sp}_{\pi}(s) &= \alpha\left[\frac{1}{2}\sum_{a \in S(s)}^{K}\left(\left(\frac{Q_{\pi}^{sp}(s,a)}{\alpha}\right)^2 - \tau\left(\frac{Q_{\pi}^{sp}(s,\cdot)}{\alpha}\right)^2\right) + \frac{1}{2}\right]\\
\pi(a|s) &= \max\left(\frac{Q_{\pi}^{sp}(s,a)}{\alpha} -\tau\left(\frac{Q_{\pi}^{sp}(s,\cdot)}{\alpha}\right),0\right).
\end{aligned}
\end{eqnarray*}
\end{proof}

\subsection{Upper and Lower Bounds for Sparsemax Operation}\label{subsec:upper_lower_bnd_spmax}

In this section, we prove the lower and upper bounds of $\text{spmax}(z)$ defined in (\ref{eqn:spmax_dfn}).
We would like to mention that the proof of lower bound of (\ref{eqn:spmax_bnd}) is
provided in \cite{martins2016softmax}. 
However, we find another interesting way to prove
(\ref{eqn:spmax_bnd}) by using the Cauchy-Schwartz inequality and
the nonnegative property of a quadratic equation. 

We first prove $\max(z) \le \text{spmax}(z)$ and next prove $\text{spmax}(z) \le \max(z) + \frac{d-1}{2d}$.
Without loss of generality, we assume that $\alpha = 1$ but the original inequalities can be simply obtained by replacing $z$ with $\frac{z}{\alpha}$.
\begin{lower_bnd}
For all $z \in \mathbb{R}^{d}$, $\max(z) \le \textnormal{spmax}(z)$ holds.
\end{lower_bnd}

\begin{proof}
We prove that, for all $z$, $\text{spmax}(z)- z_{(1)} \ge 0$ where $z_{(1)} = \max(z)$ by definition.
The proof is done by simply rearranging the terms in (\ref{eqn:spmax_dfn}),
\begin{eqnarray*}
\small
\begin{aligned}
&\text{spmax}(z)- z_{(1)} \\
&= \frac{1}{2}\sum_{i = 1}^{K}\left(z_{(i)}^2 - \tau(z)^2\right) + \frac{1}{2}- z_{(1)}\\
&= \frac{1}{2}\sum_{i = 1}^{K}z_{(i)}^2 - \frac{K}{2}\left(\frac{\sum_{i=1}^{K}z_{(i)} - 1}{K}\right)^2 + \frac{1}{2}- z_{(1)}\\
&= \frac{1}{2}\sum_{i = 1}^{K}z_{(i)}^2 - \frac{1}{2K}\left(\sum_{i=1}^{K}z_{(i)} - 1\right)^2 + \frac{1}{2}- z_{(1)}\\
&= \frac{K\sum_{i = 1}^{K}z_{(i)}^2 - \left(\sum_{i=1}^{K}z_{(i)} - 1\right)^2 -2Kz_{(1)} + K}{2K}\\
&= \frac{1}{2K} \Bigg( Kz_{(1)}^2 + K\sum_{i = 2}^{K}z_{(i)}^2 \\
&-\left(z_{(1)} + \sum_{i=2}^{K}z_{(i)} - 1\right)^2 -2Kz_{(1)} + K \Bigg).
\end{aligned}
\end{eqnarray*}
The quadratic term can be decomposed as follows:
\begin{eqnarray*}
\small
\begin{aligned}
&\left(z_{(1)} + \sum_{i=2}^{K}z_{(i)} - 1\right)^2\\
& = z_{(1)}^2 + \left(\sum_{i=2}^{K}z_{(i)}\right)^2 + 1 + 2z_{(1)}\sum_{i=2}^{K}z_{(i)} - 2z_{(1)} - 2\sum_{i=2}^{K}z_{(i)}.
\end{aligned}
\end{eqnarray*}
By putting this result into the equation and rearranging them, three
terms are obtained as follows:{
\begin{eqnarray*}
%\small
\begin{aligned}
&\text{spmax}(z) - z_{(1)}\\
 &= \frac{1}{2K}\Bigg( (K-1)z_{(1)}^2 - 2z_{(1)}\left\{\sum_{i=2}^{K}z_{(i)} + K - 1\right\} \\
&+K\sum_{i = 2}^{K}z_{(i)}^2 + 2\sum_{i = 2}^{K}z_{(i)}+ K - \left(\sum_{i = 2}^{K}z_{(i)}\right)^2\Bigg).
\end{aligned}
\end{eqnarray*}
Then, $K\sum_{i = 2}^{K}z_{(i)}^2 + 2\sum_{i = 2}^{K}z_{(i)}+ K$ can be replaced with
$K\sum_{i = 2}^{K}\left(z_{(i)} + 1\right)^2- 2(K-1)\sum_{i = 2}^{K}z_{(i)}$ and
we also decompose the second term $- 2z_{(1)}\left\{\sum_{i=2}^{K}z_{(i)} + K - 1\right\}$ into two parts: $- 2z_{(1)}\left\{\sum_{i=2}^{K}(z_{(i)} +1)\right\}$  and $2z_{(1)}$,
and rearrange the equation as follows,
\begin{eqnarray*}
\small
\begin{aligned}
&=  \frac{1}{2K}\Bigg((K-1)z_{(1)}^2 - 2z_{(1)}\left\{\sum_{i=2}^{K}\left(z_{(i)} + 1\right)\right\}\\
&+K\sum_{i = 2}^{K}\left(z_{(i)} + 1\right)^2 - 2(K-1)\sum_{i = 2}^{K}z_{(i)} - \left(\sum_{i = 2}^{K}z_{(i)}\right)^2\Bigg).
\end{aligned}
\end{eqnarray*}
Again, we change $ - 2(K-1)\sum_{i = 2}^{K}z_{(i)} - \left(\sum_{i = 2}^{K}z_{(i)}\right)^2$ into
$- \left(\sum_{i = 2}^{K}(z_{(i)}+1)\right)^2 + (K-1)^2$ by adding and subtracting $(K-1)^2$ as follow,
\begin{eqnarray*}
\small
\begin{aligned}
&=\frac{1}{2K}\Bigg((K-1)z_{(1)}^2 - 2z_{(1)}\left\{\sum_{i=2}^{K}\left(z_{(i)} + 1\right)\right\}\\ \\
&+ K\sum_{i = 2}^{K}\left(z_{(i)} + 1\right)^2 - \left(\sum_{i = 2}^{K}(z_{(i)}+1)\right)^2 + (K-1)^2\Bigg).
\end{aligned}
\end{eqnarray*}
Then, the term $(K-1)z_{(1)}^2 - 2z_{(1)}\left\{\sum_{i=2}^{K}\left(z_{(i)} + 1\right)\right\}$ is reformulated
as $(K-1)\left(z_{(1)} - \frac{\sum_{i=2}^{K}\left(z_{(i)} + 1\right)}{K-1}\right)^2- (K-1)\left(\frac{\sum_{i=2}^{K}(z_{(i)+1})}{K-1}\right)^2$.
By using this reformulation, 
we can obtain following equation.
\begin{eqnarray*}
\small
\begin{aligned}
&=  \frac{(K-1)}{2K}\left[z_{(1)} - \frac{\sum_{i=2}^{K}\left(z_{(i)} + 1\right)}{K-1}\right]^2+\\
&\frac{1}{2K}\Bigg(- \frac{\left(\sum_{i=2}^{K}(z_{(i)+1})\right)^2}{K-1} + K\sum_{i = 2}^{K}\left(z_{(i)} + 1\right)^2- \left(\sum_{i = 2}^{K}(z_{(i)}+1)\right)^2\\
& + (K-1)^2\Bigg).
\end{aligned}
\end{eqnarray*}
Finally, we can obtain three terms by rearranging the above equation,
\begin{eqnarray*}
\small
\begin{aligned}
&= \frac{(K-1)}{2K}\left[z_{(1)} - \frac{\sum_{i=2}^{K}\left(z_{(i)} + 1\right)}{K-1}\right]^2\\
&+\frac{1}{2K}\Bigg(K\sum_{i = 2}^{K}\left(z_{(i)} + 1\right)^2  - K\frac{\left(\sum_{i=2}^{K}(z_{(i)}+1)\right)^2}{K-1}\Bigg) + \frac{(K-1)^2}{2K}\\
&= \frac{(K-1)}{2K}\left[z_{(1)} - \frac{\sum_{i=2}^{K}\left(z_{(i)} + 1\right)}{K-1}\right]^2\\
&+\frac{K-1}{2}\left[\sum_{i = 2}^{K}\frac{\left(z_{(i)} + 1\right)^2}{K-1}  - \left(\sum_{i=2}^{K}\frac{(z_{(i)}+1)}{K-1}\right)^2\right]+ \frac{(K-1)^2}{2K}
\end{aligned}
\end{eqnarray*}}
where the first and third terms are quadratic and always nonnegative. 
The second term is also always nonnegative by the Cauchy-Schwartz
inequality.  
The Cauchy-Schwartz inequality is written as $(\bold{p}^{\intercal}\bold{q})^2\le||\bold{p}||^2||\bold{q}||^2$. Let $z_{2:K} = [z_{(2)},\cdots,z_{(K)}]^{\intercal}$, then,
by setting $\bold{p} = z_{2:K} + \bold{1}$ and $\bold{q} =
\frac{1}{K-1}\bold{1}$ where $\bold{1}$ is a $K-1$ dimensional
vector of ones, it can be shown that the second term is nonnegative.
Therefore, $\text{spmax}(z) - z_{(1)}$ is always nonnegative for all
$z$ since three remaining terms are always nonnegative, completing
the proof.
\end{proof}

Now, we prove the upper bound of sparsemax operation.
\begin{upper_bnd}
For all $z \in \mathbb{R}^{d}$, $\textnormal{spmax}(z) \le \max(z) + \frac{d-1}{2d}$ holds.
\end{upper_bnd}

\begin{proof}
First, we decompose the summation of (\ref{eqn:spmax_dfn}) into two terms as follows:
\begin{eqnarray*}
\begin{aligned}
&\text{spmax}(z) =\frac{1}{2}\sum_{i = 1}^{K}\left(z_{(i)}^2 - \tau(z)^2\right) + \frac{1}{2} \\
&=\frac{1}{2}\sum_{i = 1}^{K}\left(z_{(i)} - \tau(z)\right)\left(z_{(i)} + \tau(z)\right) + \frac{1}{2}\\
&\leq\frac{1}{2}\sum_{i = 1}^{K}p_{i}^{*}(z)\left(z_{(i)} + \tau(z)\right) + \frac{1}{2}\\
&=\frac{1}{2}\sum_{i = 1}^{K}p_{i}^{*}(z)z_{(i)} + \frac{\tau(z)}{2}\sum_{i = 1}^{K}p_{i}^{*}(z) + \frac{1}{2}\\
&=\frac{1}{2}\sum_{i = 1}^{K}p_{i}^{*}(z)z_{(i)} + \frac{\tau(z)}{2} + \frac{1}{2}\\
&=\frac{1}{2}\sum_{i = 1}^{K}p_{i}^{*}(z)z_{(i)} + \frac{1}{2}\sum_{i = 1}^{K}\frac{z_{(i)}}{K} - \frac{1}{2K} + \frac{1}{2}
\end{aligned}
\end{eqnarray*}
where $p_{i}^{*} = \max(z_{(i)} - \tau(z),0)$ which is the optimal solution of the simplex projection problem (\ref{eqn:simplex_prob}) and $\sum_{i=1}^{K}p_{i}^{*}(z) = 1$ by definition.
Now, we use the fact that, for every $p$ on $d-1$ dimensional simplex, $\sum_{i}^{d}p_{i}z_{i} \leq \max(z)$ for all $z \in \mathbb{R}^{d}$.
By using this property, as $p^{*}(z)$ and $\frac{1}{K}\bold{1}$ are on the probability simplex, following inequality is induced,
\begin{eqnarray*}
\begin{aligned}
&\text{spmax}(z)
=\frac{1}{2}\sum_{i = 1}^{K}p_{i}^{*}(z)z_{(i)} + \frac{1}{2}\sum_{i = 1}^{K}\frac{z_{(i)}}{K} - \frac{1}{2K} + \frac{1}{2}\\
&\leq \frac{1}{2}\max(z) + \frac{1}{2}\max(z) - \frac{1}{2K} + \frac{1}{2}
\leq \max(z) - \frac{1}{2K} + \frac{1}{2} \\
&\leq \max(z) - \frac{1}{2d} + \frac{1}{2}
\end{aligned}
\end{eqnarray*}
where $d \ge K$ by definition of $K$.
Therefore, $\textnormal{spmax}(z) \le \max(z) + \frac{d-1}{2d}$ holds.
\end{proof}

\subsection{Comparison to \textit{Log-Sum-Exp}}\label{subsec:comp_lse_spmax}

We explain the error bounds for the \textit{log-sum-exp} operation and
compare it to the bounds of the sparsemax operation. 
The \textit{log-sum-exp} operation has widely known bounds,
\begin{eqnarray*}
\begin{aligned}
\max(z)\leq\text{logsumexp}(z)\leq\max(z) + \log(d).
\end{aligned}
\end{eqnarray*}
We would like to note that \textit{sparsemax} has tighter bounds than \textit{log-sum-exp}
as it is always satisfied that, for all $d > 1$, $\frac{d-1}{2d}\leq\log(d)$.
Intuitively, the approximation error of \textit{log-sum-exp} increases as the dimension of input space increases.
However, the approximation error of \textit{sparsemax} approaches to $\frac{1}{2}$ as the dimension of input space goes infinity.
This fact plays a crucial role in comparing performance error bounds of the sparse MDP and soft MDP.

\subsection{Causal Sparse Tsallis Entropy}\label{subsec:csp_ent}
The following proof shows that $W(\pi)$ is equivalent to the discounted expected sum of special case of Tsallis entropy when $q = 2$ and $k=\frac{1}{2}$.

\begin{proof}[Proof of Theorem \ref{thm:causal_tsallis_entropy}]
The proof is simply done by rewriting our regularization as follows:
\begin{eqnarray*}
\small
\begin{aligned}
&W(\pi) \\
&= \mathbb{E}\left[\sum_{t=0}^{\infty}\gamma^{t}\frac{1}{2}(1 - \pi(a_t|s_t))\middle|\pi,d,T\right]\\
&=  \sum_{s,a}\frac{1}{2}(1 - \pi(a|s))\mathbb{E}\left[\sum_{t=0}^{\infty}\gamma^{t}\mathbbm{1}_{\{s_t=s,a_t=a\}}\middle|\pi,d,T\right]\\
&=  \sum_{s,a}\frac{1}{2}(1 - \pi(a|s))\rho_{\pi}(s,a)\\
&= \sum_{s}\rho_{\pi}(s)\sum_{a}\frac{1}{2}(1 - \pi(a|s))\pi(a|s)\\
&= \sum_{s}\rho_{\pi}(s) \frac{1}{2}(\sum_{a} \pi(a|s)-\sum_{a} \pi(a|s)^{2})\\
&= \sum_{s}\rho_{\pi}(s) \frac{1}{2}(1-\sum_{a} \pi(a|s)^{2})\\
&=\sum_{s}S_{2,\frac{1}{2}}(\pi(\cdot|s))\rho_{\pi}(s) = \mathbb{E}_{\pi}\left[S_{2,\frac{1}{2}}(\pi(\cdot|s))\right].
\end{aligned}
\end{eqnarray*}
\end{proof}

\subsection{Convergence and Optimality of Sparse Value Iteration}\label{subsec:conv_opt_sp_value}

In this section, the monotonicity, discounting property, contraction of sparse Bellman operation $U^{sp}$ are proved.

\begin{proof}[Proof of Lemma \ref{lem:mon}]
In \cite{martins2016softmax}, the monotonicity of (\ref{eqn:spmax_dfn}) is proved.
Then, the monotonicity of $U^{sp}$ can be proved using (\ref{eqn:spmax_dfn}).
Let $x$ and $y$ are given such that $x \leq y$.
{
Then,
\begin{eqnarray*}
\small
\begin{aligned}
&\frac{r(s,a) + \gamma\sum_{s'} x(s')T(s'|s,a)}{\alpha} \leq \frac{r(s,a) + \gamma\sum_{s'} y(s')T(s'|s,a)}{\alpha}
\end{aligned}
\end{eqnarray*}
where $T(s'|s,a)$ is a transition probability which is always nonnegative.
Since the sparsemax operation is monotone, the following inequality is induced
\begin{eqnarray*}
\small
\begin{aligned}
&\alpha\text{spmax}\left(\frac{r(s,a) + \gamma\sum_{s'} x(s')T(s'|s,a)}{\alpha}\right) \\
&\leq \alpha\text{spmax}\left(\frac{r(s,a) + \gamma\sum_{s'} y(s')T(s'|s,a)}{\alpha}\right).
\end{aligned}
\end{eqnarray*}
Finally, we can obtain
\begin{eqnarray*}
\small
\begin{aligned}
&\therefore\;\; U^{sp}(x)\leq U^{sp}(y).
\end{aligned}
\end{eqnarray*}}
\end{proof}

\begin{proof}[Proof of Lemma \ref{lem:dis}]
In \cite{martins2016softmax}, it is shown that for $c \in \mathbb{R}$ and $x \in \mathbb{R}^{|\mathcal{S}|}$, $\text{spmax}(x+c\bold{1}) = \text{spmax}(x) + c\bold{1}$.
Using this property,
\begin{eqnarray*}
\small
\begin{aligned}
&U^{sp}(x + c\bold{1})(s)\\
 &= \alpha\text{spmax}\left(\frac{r(s,a) + \gamma\sum_{s'} (x(s')+c)T(s'|s,a)}{\alpha}\right)\\
&= \alpha\text{spmax}\left(\frac{r(s,a) + \gamma\sum_{s'}x(s')T(s'|s,a) + \gamma c \sum_{s'}T(s'|s,a)}{\alpha}\right)\\
&= \alpha\text{spmax}\left(\frac{r(s,a) + \gamma\sum_{s'}x(s')T(s'|s,a)}{\alpha} + \frac{\gamma c}{\alpha}\right)\\
&= \alpha\text{spmax}\left(\frac{r(s,a) + \gamma\sum_{s'}x(s')T(s'|s,a)}{\alpha}\right) + \gamma c\\
&\therefore\;\; U^{sp}(x+c\bold{1}) = U^{sp}(x)+\gamma c\bold{1}.
\end{aligned}
\end{eqnarray*}
\end{proof}

\begin{proof}[Proof of Lemma \ref{lem:fixed}]
First, we prove that $U^{sp}$ is a $\gamma$-contraction mapping with
respect to $d_{max}$. 
Without loss of generality, the proof is discussed for a general
function 
$\phi : \mathbb{R}^{|\mathcal{S}|} \rightarrow \mathbb{R}^{|\mathcal{S}|}$ 
with discounting and monotone properties.

Let $d_{max}(x,y) = M$. Then, $y-M\bold{1}\leq x\leq y + M\bold{1}$ is satisfied.
By monotone and discounting properties, the following inequality
between mappings $\phi(x)$ and $\phi(y)$ is established. 
\begin{eqnarray*}
\begin{aligned}
\phi(y)-\gamma M\bold{1}\leq \phi(x)\leq \phi(y) + \gamma M\bold{1},
\end{aligned}
\end{eqnarray*}
where $\gamma$ is a discounting factor of $\phi$.
From this inequality, $d_{max}(\phi(x),\phi(y)) \leq \gamma M = \gamma d_{max}(x,y)$ and $\gamma \in (0,1)$.
Therefore, $\phi$ is a $\gamma$-contraction mapping.
In our case, $U^{sp}$ is a $\gamma$-contraction mapping.

As $\mathbb{R}^{|\mathcal{S}|}$ and $d_{max}(x,y)$ are a non-empty complete metric space, 
by Banach fixed-point theorem, a $\gamma$-contraction mapping $U^{sp}$ has a unique fixed point.
\end{proof}

Using Lemma \ref{lem:mon}, Lemma \ref{lem:dis}, and Lemma
\ref{lem:fixed}, we can prove the convergence and optimality of
sparse value iteration. 

\begin{proof}[Proof of Theorem \ref{thm:opt_sps_mdp}]
Sparse value iteration converges into a fixed point of $U^{sp}$ by
the contraction property.
Let $x_{*}$ be a fixed point of $U^{sp}$ and, by definition of
$U^{sp}$, $x_{*}$ is the point that satisfies the sparse Bellman
equation, i.e. $x_{*} = U^{sp}(x_{*})$. 
Hence, by Theorem \ref{thm:sp_bellman_eqn}, $x_{*}$ satisfies
necessity conditions of the optimal solution.
By the Banach fixed point theorem, $x_{*}$ is a unique point which
satisfies necessity conditions of optimal solution.
In particular, $x_{*} = U^{sp}(x_{*})$ is precisely equivalent to the sparse Bellman equation.
In other words, there is no other point that satisfies the sparse
Bellman equation. 
Therefore, $x_{*}$ is the optimal value of a sparse MDP.
\end{proof}

\subsection{Performance Error Bounds for Sparse Value Iteration}\label{sec:error_bnd}

In this section, we prove the performance error bounds for sparse
value iteration and soft value iteration. 
We first show that the optimal values of a sparse MDP and a soft MDP
are greater than that of the original MDP. 

\begin{proof}[Proof of Lemma \ref{lem:operation}]
We first prove the inequality of the sparse Bellman operation
\begin{equation*}
U^{n}(x) \leq (U^{sp})^{n}(x),\;\;x_{*} \leq x_{*}^{sp}.
\end{equation*}
This inequality can be proven by the mathematical induction.
When $n=1$, the inequality is proven as follows:
\begin{eqnarray*}
&\max_{a'}\left(r(s,a') + \gamma \sum_{s'}x(s')T(s'|s,a')\right)\\
&\leq \text{spmax}\left(r(s,\cdot) + \gamma \sum_{s'}x(s')T(s'|s,\cdot)\right)\\
&(\because \max(z) \leq \text{spmax}(z)).
\end{eqnarray*}
Therefore,
\begin{eqnarray*}
&U(x) \leq U^{sp}(x).
\end{eqnarray*}
For some positive integer $k$, let us assume that $U^{k}(x) \leq (U^{sp})^{k}(x)$ holds for every $x \in \mathbb{R}^{|\mathcal{S}|}$.
Then, when $n = k+1$, 
\begin{eqnarray*}
U^{k+1}(x) &=& U^{k}(U(x)) \\
&\leq& (U^{sp})^{k}(U(x))\;\; (\because U^{k}(x) \leq (U^{sp})^{k}(x)) \\
 &\leq& (U^{sp})^{k}(U^{sp}(x))\;\; (\because U(x) \leq U^{sp}(x))\\
 &=& (U^{sp})^{k+1}(x).
\end{eqnarray*}
Therefore, by mathematical induction, it is satisfied $U^{n}(x) \leq (U^{sp})^{n}(x)$ for every positive integer $n$.
Then, the inequality of the fixed points of $U$ and $U^{sp}$ can be obtained by $n\rightarrow \infty$,
\begin{equation*}
x_{*} \leq x^{sp}_{*}
\end{equation*}
where $*$ indicates the fixed point.
The above arguments also hold when $U^{sp}$ and \textit{sparsemax} are replaced with $U^{soft}$ and \textit{log-sum-exp} operation, respectively.
\end{proof}

Before showing the performance error bounds, 
the upper bounds of $W(\pi)$ and $H(\pi)$ are proved first.

%\begin{lemma} \label{lem:reg_bnd}
%$W(\pi)$ and $H(\pi)$ have following upper bounds:
%{\small
%\begin{equation*}
%W(\pi) \leq \frac{1}{1-\gamma}\frac{|\mathcal{A}|-1}{2|\mathcal{A}|},\;\;
%H(\pi) \leq \frac{\log(|\mathcal{A}|)}{1-\gamma},
%\end{equation*}}
%where $|\mathcal{A}|$ is the cardinality of the action space $\mathcal{A}$.
%\end{lemma}

\begin{proof} [Proof of Lemma \ref{lem:reg_bnd}]
For $W(\pi)$,
\begin{eqnarray*}
\small
\begin{aligned}
&W(\pi) = \sum_{s}\rho_{\pi}(s)\sum_{a}\frac{1}{2}(1 - \pi(a|s))\pi(a|s)\\
&\leq \sum_{s}\rho_{\pi}(s)\frac{|\mathcal{A}|-1}{2|\mathcal{A}|}\;\;(\because \sum_{a}\frac{1}{2}(1 - \pi(a|s))\pi(a|s) \leq \frac{|\mathcal{A}|-1}{2|\mathcal{A}|})\\
&= \frac{1}{1-\gamma}\frac{|\mathcal{A}|-1}{2|\mathcal{A}|} \;\;(\because \sum_{s}\rho_{\pi}(s) = \frac{1}{1-\gamma}).
\end{aligned}
\end{eqnarray*}
The inequality that $\sum_{a}\frac{1}{2}(1 - \pi(a|s))\pi(a|s) \leq \frac{|\mathcal{A}|-1}{2|\mathcal{A}|}$ can be obtained by finding the point where the derivative of $\frac{1}{2}(1-x)x$ is zero.
Similarly, for $H(\pi)$,
\begin{eqnarray*}
\small
\begin{aligned}
&H(\pi) = \mathbb{E}\left[\sum_{t=0}^{\infty}\gamma^{t}-\log(\pi(a_t|s_t))\middle|\pi,d,T\right]\\ 
&=  \sum_{s,a}-\log(\pi(a|s))\mathbb{E}\left[\sum_{t=0}^{\infty}\gamma^{t}\mathbbm{1}_{\{s_t=s,a_t=a\}}\middle|\pi,d,T\right]\\
&=  \sum_{s,a}-\log(\pi(a|s))\rho_{\pi}(s,a)\\
&= \sum_{s}\rho_{\pi}(s)\sum_{a}-\log(\pi(a|s))\pi(a|s)\\
&\leq \sum_{s}\rho_{\pi}(s)\log(|\mathcal{A}|)\;\;(\because \sum_{a}-\log(\pi(a|s))\pi(a|s) \leq \log(|\mathcal{A}|))\\
&= \frac{1}{1-\gamma}\log(|\mathcal{A}|) \;\;(\because \sum_{s}\rho_{\pi}(s) = \frac{1}{1-\gamma}).
\end{aligned}
\end{eqnarray*}
The inequality that $\sum_{a}-\log(\pi(a|s))\pi(a|s) \leq \log(|\mathcal{A}|)$ also can be obtained by finding the point where the derivative of $-x\log(x)$ is zero.
\end{proof}

Using Lemma \ref{lem:operation} and Lemma \ref{lem:reg_bnd}, the error
bounds of sparse and soft value iterations can be proved. 
%%%%%%%%%%%%%%
%%% Bounds for Sparse Value Iteration
%%%%%%%%%%%%%%

%\begin{theorem}\label{thm:bnd_sps_mdp}
%Following inequalities hold:
%{\small
%\begin{equation*}
%\mathbb{E}_{\pi^{*}}(\mathbf{r}(s,a)) - \frac{\alpha}{1-\gamma}\frac{|\mathcal{A}|-1}{2|\mathcal{A}|} \leq \mathbb{E}_{\pi^{sp}}(\mathbf{r}(s,a)) \leq \mathbb{E}_{\pi^{*}}(\mathbf{r}(s,a)), 
%\end{equation*}
%}
%where $\pi^{*}$ and $\pi^{sp}$ are the optimal policy obtained by the
%original MDP and a sparse MDP, respectively, and $|\mathcal{A}|$ is
%the cardinality of the action space. 
%\end{theorem}

\begin{proof}[Proof of Theorem \ref{thm:bnd_sps_mdp}]
Let $\pi_{*}$ be the optimal policy of the original MDP, 
where the problem is defined as $\max_{\pi} \mathbb{E}_{\pi}(\mathbf{r}(s,a))$.
\begin{equation*}
\mathbb{E}_{\pi^{sp}_{*}}(\mathbf{r}(s,a)) \leq \max_{\pi} \mathbb{E}_{\pi}(\mathbf{r}(s,a)) = \mathbb{E}_{\pi_{*}}(\mathbf{r}(s,a)).
\end{equation*}
The rightside inequality is by the definition of optimality.
Before proving the leftside inequality, we first derive the following
inequality from Lemma \ref{lem:operation}: 
\begin{equation}\label{eqn:sps_value_ieq}
V_{*} \leq V^{sp}_{*},
\end{equation}
where $*$ indicates an optimal value.
Since the fixed points of $U$ and $U^{sp}$ are the optimal solutions
of the original MDP and sparse MDP, respectively,
(\ref{eqn:sps_value_ieq}) can be derived from Lemma \ref{lem:operation}. 
The leftside inequality is proved using (\ref{eqn:sps_value_ieq}) as follows:
\begin{eqnarray*}
\begin{aligned}
&\mathbb{E}_{\pi_{*}}(\mathbf{r}(s,a)) = d^{\intercal}V_{*} \\
&\leq d^{\intercal}V^{sp}_{*} = J^{sp}_{*} = \mathbb{E}_{\pi_{*}^{sp}}(\mathbf{r}(s,a)) + \alpha W(\pi_{*}^{sp})\\
&\leq \mathbb{E}_{\pi_{*}^{sp}}(\mathbf{r}(s,a)) + \frac{\alpha}{1-\gamma}\frac{|\mathcal{A}|-1}{2|\mathcal{A}|}\;\;
(\because\;\text{Lemma \ref{lem:reg_bnd}}).
\end{aligned}
\end{eqnarray*}
\end{proof}

%%%%%%%%%%%%%%
%%% Bounds for Soft Max and Soft Value Iteration and Comparison Results
%%%%%%%%%%%%%%

%\begin{theorem}\label{thm:bnd_soft_mdp}
%Following inequalities hold:
%{\small
%\begin{equation*}
%\mathbb{E}_{\pi^{*}}(\mathbf{r}(s,a)) - \frac{\alpha}{1-\gamma}\log(|\mathcal{A}|) \leq \mathbb{E}_{\pi^{soft}}(\mathbf{r}(s,a)) \leq \mathbb{E}_{\pi^{*}}(\mathbf{r}(s,a)) 
%\end{equation*}
%}
%where $\pi^{*}$ and $\pi^{soft}$ are the optimal policies obtained by
%the original MDP and a soft MDP, respectively, and $|\mathcal{A}|$ is
%the cardinality of the action space. 
%\end{theorem}

\begin{proof}[Proof of Theorem \ref{thm:bnd_soft_mdp}]
Let $\pi_{*}$ be the optimal policy of the original MDP which is
defined as $\max_{\pi} \mathbb{E}_{\pi}(\mathbf{r}(s,a))$.
The rightside inequality is by the definition of optimality.
\begin{equation*}
\mathbb{E}_{\pi^{soft}_{*}}(\mathbf{r}(s,a)) \leq \max_{\pi} \mathbb{E}_{\pi}(\mathbf{r}(s,a)) = \mathbb{E}_{\pi_{*}}(\mathbf{r}(s,a)).
\end{equation*}
Before proving the leftside inequality, we first derive following
inequality from Lemma \ref{lem:operation}: 
\begin{equation}\label{eqn:soft_value_ieq}
V_{*} \leq V^{soft}_{*}
\end{equation}
where $*$ indicates an optimal solution.
Then, the proof of the leftside inequality is done by using (\ref{eqn:soft_value_ieq}) as follows:
\begin{eqnarray*}
\begin{aligned}
&\mathbb{E}_{\pi_{*}}(\mathbf{r}(s,a)) = d^{\intercal}V_{*} \\
&\leq d^{\intercal}V^{soft}_{*} = J^{soft}_{*} = \mathbb{E}_{\pi_{*}^{soft}}(\mathbf{r}(s,a)) + \alpha H(\pi_{*}^{soft})\\
&\leq \mathbb{E}_{\pi_{*}^{soft}}(\mathbf{r}(s,a)) + \frac{\alpha}{1-\gamma}\log(|\mathcal{A}|)\;\;
(\because\;\text{Lemma \ref{lem:reg_bnd}}).
\end{aligned}
\end{eqnarray*}
\end{proof}

\section{}\label{sec:exp_full}
In this section, we present the full experimental results of
reinforcement learning with a continuous action space.
We performe experiments on \textit{Inverted Pendulum} and
\textit{Reacher} and 28 algorithms are tested including our sparse
exploration method and sparse Bellman update rule. 

\begin{table*}[t]
\centering
\begin{tabular}{|l|l|l|l|l|l|}
\hline
The Number of Action & 3 & 101 & 1001 & 2001 & Average \\ \hline
Sparse+SparseBellman-1 & 1000.0 & 996.8 & 1000.0 & 1000.0 & \textbf{999.2} \\ \hline 
Sparse+SparseBellman-0.1 & 1000.0 & 933.1 & 668.2 & 1000.0 & 900.3 \\ \hline 
Sparse+SparseBellman-0.01 & 1000.0 & 992.1 & 1000.0 & 1000.0 & \textbf{998.0} \\ \hline 
Sparse+SoftBellman-1 & 1000.0 & 1000.0 & 925.2 & 1000.0 & 981.3 \\ \hline 
Sparse+SoftBellman-0.1 & 1000.0 & 1000.0 & 1000.0 & 1000.0 & \textbf{1000.0} \\ \hline 
Sparse+SoftBellman-0.01 & 782.7 & 988.6 & 775.8 & 1000.0 & 886.8 \\ \hline 
Sparse+Bellman-1 & 1000.0 & 1000.0 & 919.7 & 715.3 & 908.7 \\ \hline 
Sparse+Bellman-0.1 & 980.2 & 745.5 & 1000.0 & 1000.0 & 931.4 \\ \hline 
Sparse+Bellman-0.01 & 1000.0 & 1000.0 & 1000.0 & 1000.0 & \textbf{1000.0} \\ \hline 
Soft+SparseBellman-1 & 673.9 & 835.5 & 53.3 & 1000.0 & 640.7 \\ \hline 
Soft+SparseBellman-0.1 & 688.0 & 1000.0 & 938.0 & 904.2 & 882.6 \\ \hline 
Soft+SparseBellman-0.01 & 993.0 & 1000.0 & 736.4 & 1000.0 & 932.3 \\ \hline 
Soft+SoftBellman-1 & 939.6 & 738.3 & 506.3 & 943.2 & 781.9 \\ \hline 
Soft+SoftBellman-0.1 & 1000.0 & 1000.0 & 1000.0 & 681.6 & 920.4 \\ \hline 
Soft+SoftBellman-0.01 & 1000.0 & 974.8 & 1000.0 & 1000.0 & 993.7 \\ \hline 
Soft+Bellman-1 & 668.9 & 621.5 & 668.7 & 643.2 & 650.6 \\ \hline 
Soft+Bellman-0.1 & 1000.0 & 1000.0 & 1000.0 & 1000.0 & \textbf{1000.0} \\ \hline 
Soft+Bellman-0.01 & 977.6 & 1000.0 & 1000.0 & 1000.0 & 994.4 \\ \hline 
EpsGrdy+SparseBellman-1 & 479.8 & 669.0 & 344.5 & 678.1 & 542.9 \\ \hline 
EpsGrdy+SparseBellman-0.1 & 668.1 & 1000.0 & 351.1 & 666.6 & 671.4 \\ \hline 
EpsGrdy+SparseBellman-0.01 & 1000.0 & 124.6 & 477.5 & 667.8 & 567.5 \\ \hline 
EpsGrdy+SoftBellman-1 & 940.3 & 684.9 & 658.3 & 505.6 & 697.3 \\ \hline 
EpsGrdy+SoftBellman-0.1 & 338.5 & 376.8 & 1000.0 & 1000.0 & 678.8 \\ \hline 
EpsGrdy+SoftBellman-0.01 & 551.5 & 652.8 & 735.2 & 677.9 & 654.3 \\ \hline 
EpsGrdy+Bellman-1 & 332.7 & 1000.0 & 1000.0 & 369.8 & 675.6 \\ \hline 
EpsGrdy+Bellman-0.1 & 1000.0 & 618.7 & 1000.0 & 771.7 & 847.6 \\ \hline 
EpsGrdy+Bellman-0.01 & 462.6 & 676.5 & 698.0 & 48.1 & 471.3 \\ \hline 
DDPG                                        & \multicolumn{4}{l|}{253.1}&253.1 \\ \hline
\end{tabular}
\caption{Expected return of \textit{Inverted Pendulum}. Top five performances are marked in bold. }
\end{table*}

\begin{table*}[t]
\centering
\begin{tabular}{|l|l|l|l|l|}
\hline
The Number of Action & 3 & 101 & 1001 & 2001 \\ \hline
Sparse+SparseBellman-1 & 1164 & 692 & {742} & 864 \\ \hline 
Sparse+SparseBellman-0.1 & 1060 & 2923 & 3998 & 599 \\ \hline 
Sparse+SparseBellman-0.01 & 685 & 1431 & 1010 & 811 \\ \hline 
Sparse+SoftBellman-1 & 863 & 1316 & 1698 & 657 \\ \hline 
Sparse+SoftBellman-0.1 & 914 & 901 & 857 & 802 \\ \hline 
Sparse+SoftBellman-0.01 & 3342 & 907 & 3930 & {522} \\ \hline 
Sparse+Bellman-1 & 879 & 668 & 2337 & 3137 \\ \hline 
Sparse+Bellman-0.1 & 937 & 3925 & 773 & 1030 \\ \hline 
Sparse+Bellman-0.01 & 999 & {329} & 962 & 962 \\ \hline 
Soft+SparseBellman-1 & 3789 & 3416 & 3996 & 2684 \\ \hline 
Soft+SparseBellman-0.1 & 3844 & 2835 & 1494 & 2771 \\ \hline 
Soft+SparseBellman-0.01 & 854 & 545 & 3814 & 999 \\ \hline 
Soft+SoftBellman-1 & 1885 & 3666 & 3994 & 3912 \\ \hline 
Soft+SoftBellman-0.1 & 869 & 780 & 787 & 3871 \\ \hline 
Soft+SoftBellman-0.01 & {533} & 1241 & 2565 & 3020 \\ \hline 
Soft+Bellman-1 & 3898 & 3947 & 3978 & 3758 \\ \hline 
Soft+Bellman-0.1 & 876 & 1923 & 954 & 807 \\ \hline 
Soft+Bellman-0.01 & 1419 & 689 & 755 & 1265 \\ \hline 
EpsGrdy+SparseBellman-1 & 3978 & 3993 & 4000 & 3863 \\ \hline 
EpsGrdy+SparseBellman-0.1 & 3895 & 2449 & 4000 & 3910 \\ \hline 
EpsGrdy+SparseBellman-0.01 & 3437 & 4000 & 3962 & 3777 \\ \hline 
EpsGrdy+SoftBellman-1 & 2959 & 3919 & 3715 & 4000 \\ \hline 
EpsGrdy+SoftBellman-0.1 & 3997 & 3969 & 3037 & 2509 \\ \hline 
EpsGrdy+SoftBellman-0.01 & 3976 & 3936 & 3785 & 3784 \\ \hline 
EpsGrdy+Bellman-1 & 4000 & 2603 & 1093 & 3969 \\ \hline 
EpsGrdy+Bellman-0.1 & 2584 & 3897 & 3160 & 3846 \\ \hline 
EpsGrdy+Bellman-0.01 & 3891 & 3699 & 3905 & 3993 \\ \hline 
\end{tabular}
\caption{The number of episodes required to reach the threshold return, 980.}
\end{table*}

\begin{table*}[t]
\centering
\begin{tabular}{|l|l|l|l|l|l|}
\hline
The Number of Action & 9 & 121 & 961 & 2601 & Average \\ \hline
Sparse+SparseBellman-1 & -7.7 & -7.8 & -10.1 & -11.5 & -9.3 \\ \hline 
Sparse+SparseBellman-0.1 & -11.3 & -5.7 & -5.4 & -5.5 & \textbf{-7.0} \\ \hline 
Sparse+SparseBellman-0.01 & -11.3 & -8.7 & -8.6 & -6.3 & -8.7 \\ \hline 
Sparse+SoftBellman-1 & -7.6 & -10.5 & -11.5 & -10.0 & -9.9 \\ \hline 
Sparse+SoftBellman-0.1 & -10.4 & -5.8 & -5.5 & -9.3 & -7.8 \\ \hline 
Sparse+SoftBellman-0.01 & -11.2 & -6.4 & -8.9 & -6.4 & -8.2 \\ \hline 
Sparse+Bellman-1 & -7.6 & -7.7 & -5.7 & -10.2 & -7.8 \\ \hline 
Sparse+Bellman-0.1 & -10.8 & -5.5 & -5.4 & -5.8 & \textbf{-6.9} \\ \hline 
Sparse+Bellman-0.01 & -11.6 & -5.9 & -5.9 & -9.4 & -8.2 \\ \hline 
Soft+SparseBellman-1 & -52.0 & -48.0 & -29.6 & -39.3 & -42.2 \\ \hline 
Soft+SparseBellman-0.1 & -7.4 & -22.4 & -20.8 & -25.5 & -19.0 \\ \hline 
Soft+SparseBellman-0.01 & -11.1 & -5.5 & -5.5 & -9.2 & -7.8 \\ \hline 
Soft+SoftBellman-1 & -52.2 & -43.1 & -46.8 & -44.1 & -46.5 \\ \hline 
Soft+SoftBellman-0.1 & -7.5 & -22.5 & -23.9 & -32.9 & -21.7 \\ \hline 
Soft+SoftBellman-0.01 & -11.6 & -5.7 & -5.5 & -7.6 & -7.6 \\ \hline 
Soft+Bellman-1 & -51.4 & -51.7 & -44.2 & -41.2 & -47.1 \\ \hline 
Soft+Bellman-0.1 & -7.1 & -10.0 & -26.7 & -27.5 & -17.8 \\ \hline 
Soft+Bellman-0.01 & -11.3 & -5.3 & -5.3 & -10.2 & -8.0 \\ \hline 
EpsGrdy+SparseBellman-1 & -11.2 & -7.6 & -5.6 & -6.2 & -7.6 \\ \hline 
EpsGrdy+SparseBellman-0.1 & -11.2 & -5.9 & -5.8 & -6.1 & -7.2 \\ \hline 
EpsGrdy+SparseBellman-0.01 & -10.9 & -5.9 & -5.5 & -6.0 & \textbf{-7.1} \\ \hline 
EpsGrdy+SoftBellman-1 & -10.5 & -5.9 & -5.7 & -6.1 & \textbf{-7.0} \\ \hline 
EpsGrdy+SoftBellman-0.1 & -11.3 & -5.7 & -5.6 & -6.2 & \textbf{-7.2} \\ \hline 
EpsGrdy+SoftBellman-0.01 & -10.8 & -6.2 & -12.1 & -9.5 & -9.6 \\ \hline 
EpsGrdy+Bellman & -10.8 & -6.5 & -5.7 & -6.5 & -7.4 \\ \hline 
EpsGrdy+Bellman & -11.1 & -6.2 & -5.9 & -5.9 & -7.3 \\ \hline 
EpsGrdy+Bellman & -10.6 & -9.4 & -8.4 & -6.5 & -8.7 \\ \hline 
DDPG                                        & \multicolumn{4}{l|}{  -10.1}&-10.1 \\ \hline
\end{tabular}
\caption{Expected return of \textit{Reacher}. Top five performances are marked in bold.}
\end{table*}

\begin{table*}[]
\centering
\begin{tabular}{|l|l|l|l|l|}
\hline
The Number of Action & 9 & 121 & 961 & 2601 \\ \hline
Sparse+SparseBellman-1 & 9193 & 9648 & 7275 & 9065 \\ \hline 
Sparse+SparseBellman-0.1 & 9791 & 5837 & 5779 & {6851} \\ \hline 
Sparse+SparseBellman-0.01 & 9783 & 6456 & 6631 & 7941 \\ \hline 
Sparse+SoftBellman-1 & 9126 & 9834 & 7603 & 8503 \\ \hline 
Sparse+SoftBellman-0.1 & 9779 & 5449 & 5642 & 7509 \\ \hline 
Sparse+SoftBellman-0.01 & 9795 & 5011 & 7260 & 7768 \\ \hline 
Sparse+Bellman-1 & 9073 & 9619 & 5646 & 8371 \\ \hline 
Sparse+Bellman-0.1 & 9756 & 5366 & {5338} & 6936 \\ \hline 
Sparse+Bellman-0.01 & 9797 & 5204 & 6525 & 7965 \\ \hline 
Soft+SparseBellman-1 & 10000 & 10000 & 10000 & 10000 \\ \hline 
Soft+SparseBellman-0.1 & 8801 & 9998 & 10000 & 10000 \\ \hline 
Soft+SparseBellman-0.01 & 9783 & 4988 & 5934 & 8774 \\ \hline 
Soft+SoftBellman-1 & 10000 & 10000 & 10000 & 10000 \\ \hline 
Soft+SoftBellman-0.1 & 8810 & 9999 & 10000 & 10000 \\ \hline 
Soft+SoftBellman-0.01 & 9794 & 4597 & 5927 & 7915 \\ \hline 
Soft+Bellman-1 & 10000 & 10000 & 10000 & 10000 \\ \hline 
Soft+Bellman-0.1 & {8700} & 9999 & 10000 & 10000 \\ \hline 
Soft+Bellman-0.01 & 9790 & {4810} & 6004 & 8737 \\ \hline 
EpsGrdy+SparseBellman-1 & 9861 & 6909 & 6994 & 7977 \\ \hline 
EpsGrdy+SparseBellman-0.1 & 9850 & 6808 & 6775 & 7873 \\ \hline 
EpsGrdy+SparseBellman-0.01 & 9847 & 7079 & 6850 & 7923 \\ \hline 
EpsGrdy+SoftBellman-1 & 9850 & 6839 & 6858 & 8026 \\ \hline 
EpsGrdy+SoftBellman-0.1 & 9844 & 6918 & 6752 & 7849 \\ \hline 
EpsGrdy+SoftBellman-0.01 & 9841 & 7176 & 9803 & 8114 \\ \hline 
EpsGrdy+Bellman & 9842 & 6797 & 6933 & 8001 \\ \hline 
EpsGrdy+Bellman & 9846 & 6680 & 7051 & 7845 \\ \hline 
EpsGrdy+Bellman & 9864 & 7192 & 6928 & 7925 \\ \hline 
\end{tabular}
\caption{The number of episodes required to reach the threshold return, -6.}
\end{table*}

\bibliographystyle{IEEEtran}
\bibliography{bib_sparse_mdp}

% Generated by IEEEtran.bst, version: 1.12 (2007/01/11)
\begin{thebibliography}{10}
\providecommand{\url}[1]{#1}
\csname url@samestyle\endcsname
\providecommand{\newblock}{\relax}
\providecommand{\bibinfo}[2]{#2}
\providecommand{\BIBentrySTDinterwordspacing}{\spaceskip=0pt\relax}
\providecommand{\BIBentryALTinterwordstretchfactor}{4}
\providecommand{\BIBentryALTinterwordspacing}{\spaceskip=\fontdimen2\font plus
\BIBentryALTinterwordstretchfactor\fontdimen3\font minus
  \fontdimen4\font\relax}
\providecommand{\BIBforeignlanguage}[2]{{%
\expandafter\ifx\csname l@#1\endcsname\relax
\typeout{** WARNING: IEEEtran.bst: No hyphenation pattern has been}%
\typeout{** loaded for the language `#1'. Using the pattern for}%
\typeout{** the default language instead.}%
\else
\language=\csname l@#1\endcsname
\fi
#2}}
\providecommand{\BIBdecl}{\relax}
\BIBdecl

\bibitem{brechtel2014probabilistic}
S.~Brechtel, T.~Gindele, and R.~Dillmann, ``Probabilistic decision-making under
  uncertainty for autonomous driving using continuous pomdps,'' in \emph{17th
  International Conference on Intelligent Transportation Systems}, October
  2014, pp. 392--399.

\bibitem{ragi2013path}
S.~Ragi and E.~K.~P. Chong, ``{UAV} path planning in a dynamic environment via
  partially observable markov decision process,'' \emph{{IEEE} Trans. Aerospace
  and Electronic Systems}, vol.~49, no.~4, pp. 2397--2412, 2013.

\bibitem{hwangbo2017control}
J.~Hwangbo, I.~Sa, R.~Siegwart, and M.~Hutter, ``Control of a quadrotor with
  reinforcement learning,'' \emph{{IEEE} Robotics and Automation Letters},
  vol.~2, no.~4, pp. 2096--2103, 2017.

\bibitem{kober2013reinforce}
J.~Kober, J.~A. Bagnell, and J.~Peters, ``Reinforcement learning in robotics:
  {A} survey,'' \emph{International Journal of Robotics Research}, vol.~32,
  no.~11, pp. 1238--1274, 2013.

\bibitem{ng2000algorithms}
A.~Y. Ng and S.~J. Russell, ``Algorithms for inverse reinforcement learning,''
  in \emph{Proc. of the 7th International Conference on Machine Learning}, June
  2000, pp. 663--670.

\bibitem{Haarnoja2017}
T.~Haarnoja, H.~Tang, P.~Abbeel, and S.~Levine, ``Reinforcement learning with
  deep energy-based policies,'' in \emph{Proc. of the 34th International
  Conference on Machine Learning}, August 2017, pp. 1352--1361.

\bibitem{Heess2012}
N.~Heess, D.~Silver, and Y.~W. Teh, ``Actor-critic reinforcement learning with
  energy-based policies,'' in \emph{Proc. of the Tenth European Workshop on
  Reinforcement Learning}, June 2012, pp. 43--58.

\bibitem{schulman2017equivalence}
J.~Schulman, P.~Abbeel, and X.~Chen, ``Equivalence between policy gradients and
  soft q-learning,'' \emph{arXiv preprint arXiv:1704.06440}, 2017.

\bibitem{tokic2011value}
M.~Tokic and G.~Palm, ``Value-difference based exploration: Adaptive control
  between epsilon-greedy and softmax,'' in \emph{{KI} 2011: Advances in
  Artificial Intelligence, 34th Annual German Conference on AI}, October 2011,
  pp. 335--346.

\bibitem{vamplew2017softmax}
P.~Vamplew, R.~Dazeley, and C.~Foale, ``Softmax exploration strategies for
  multiobjective reinforcement learning,'' \emph{Neurocomputing}, vol. 263, pp.
  74--86, 2017.

\bibitem{bloem2014infinite}
M.~Bloem and N.~Bambos, ``Infinite time horizon maximum causal entropy inverse
  reinforcement learning,'' in \emph{53rd {IEEE} Conference on Decision and
  Control}, December 2014, pp. 4911--4916.

\bibitem{Tsallis1988possible}
C.~Tsallis, ``Possible generalization of boltzmann-gibbs statistics,''
  \emph{Journal of statistical physics}, vol.~52, no.~1, pp. 479--487, 1988.

\bibitem{wang2013projection}
W.~Wang and M.~A. Carreira-Perpin{\'a}n, ``Projection onto the probability
  simplex: An efficient algorithm with a simple proof, and an application,''
  \emph{arXiv preprint arXiv:1309.1541}, 2013.

\bibitem{smart1980fixed}
D.~R. Smart, \emph{Fixed point theorems}.\hskip 1em plus 0.5em minus
  0.4em\relax CUP Archive, 1980, vol.~66.

\bibitem{lillicrap2015continuous}
T.~P. Lillicrap, J.~J. Hunt, A.~Pritzel, N.~Heess, T.~Erez, Y.~Tassa,
  D.~Silver, and D.~Wierstra, ``Continuous control with deep reinforcement
  learning,'' \emph{arXiv preprint arXiv:1509.02971}, 2015.

\bibitem{ye2000constraint}
J.~Ye, ``Constraint qualifications and necessary optimality conditions for
  optimization problems with variational inequality constraints,'' \emph{SIAM
  Journal on Optimization}, vol.~10, no.~4, pp. 943--962, 2000.

\bibitem{martins2016softmax}
A.~Martins and R.~Astudillo, ``From softmax to sparsemax: A sparse model of
  attention and multi-label classification,'' in \emph{International Conference
  on Machine Learning}, June 2016, pp. 1614--1623.

\bibitem{ziebart2010MPAs}
B.~D. Ziebart, ``Modeling purposeful adaptive behavior with the principle of
  maximum causal entropy,'' Ph.D. dissertation, Carnegie Mellon University,
  Pittsburgh, PA, USA, 2010.

\bibitem{schaul2015prioritized}
T.~Schaul, J.~Quan, I.~Antonoglou, and D.~Silver, ``Prioritized experience
  replay,'' \emph{arXiv preprint arXiv:1511.05952}, 2015.

\bibitem{watkins1992q}
C.~J. Watkins and P.~Dayan, ``Q-learning,'' \emph{Machine Learning}, vol.~8,
  no. 3-4, pp. 279--292, 1992.

\bibitem{van2016deep}
H.~van Hasselt, A.~Guez, and D.~Silver, ``Deep reinforcement learning with
  double q-learning,'' in \emph{Proc. of the Thirtieth {AAAI} Conference on
  Artificial Intelligence}, February 2016, pp. 2094--2100.

\bibitem{todorov2012mujoco}
E.~Todorov, T.~Erez, and Y.~Tassa, ``Mujoco: A physics engine for model-based
  control,'' in \emph{International Conference on Intelligent Robots and
  Systems}, October 2012, pp. 5026--5033.

\end{thebibliography}

\end{document}